\documentclass[twoside]{article}
\usepackage[accepted]{aistats2023_arxiv}

\usepackage[round]{natbib}

\usepackage{amsthm} 
\usepackage{amsmath, amssymb}
\usepackage{type1cm}
\usepackage{algorithm}
\usepackage{algpseudocode}
\usepackage{graphicx}
\usepackage{color}
\usepackage{url}
\usepackage{hyperref}
\usepackage{bm}
\usepackage[whole]{bxcjkjatype}
\usepackage{stackengine}
\def\delequal{\mathrel{\ensurestackMath{\stackon[1pt]{=}{\scriptstyle\Delta}}}}
\theoremstyle{plain}
\newtheorem{Thm}{Theorem}[section]
\theoremstyle{definition}
\newtheorem{Def}[Thm]{Definition}
\theoremstyle{remark}
\newtheorem{Rem}[Thm]{Remark}
\newtheorem{Lem}[Thm]{Lemma}

\newcommand{\B}[1]{{\bm #1}} 
\newcommand{\mac}[1]{{\mathcal #1}}
\newcommand{\mab}[1]{{\mathbb #1}}
\definecolor{customred}{RGB}{255,0,29}
\definecolor{customgreen}{RGB}{46,125,57}
\definecolor{customyellow}{RGB}{255,126,42}

\begin{document}
%

%


\twocolumn[

\aistatstitle{
Learning Dynamics in Linear VAE: Posterior Collapse Threshold, Superfluous Latent Space Pitfalls, and Speedup with KL Annealing 
}
\aistatsauthor{Yuma Ichikawa \And Koji Hukushima}

\aistatsaddress{University of Tokyo, Fujitsu Limited. \And University of Tokyo } ]

\begin{abstract}
Variational autoencoders (VAEs) face a notorious problem wherein the variational posterior often aligns closely with the prior, a phenomenon known as posterior collapse, which hinders the quality of representation learning. To mitigate this problem, an adjustable hyperparameter $\beta$ and a strategy for annealing this parameter, called KL annealing, are proposed. This study presents a theoretical analysis of the learning dynamics in a minimal VAE. It is rigorously proved that the dynamics converge to a deterministic process within the limit of large input dimensions, thereby enabling a detailed dynamical analysis of the generalization error. 
Furthermore, the analysis shows that the VAE initially learns entangled representations and gradually acquires disentangled representations. A fixed-point analysis of the deterministic process reveals that when $\beta$ exceeds a certain threshold, posterior collapse becomes inevitable regardless of the learning period. Additionally, the superfluous latent variables for the data-generative factors lead to overfitting of the background noise; this adversely affects both generalization and learning convergence. The analysis further unveiled that appropriately tuned KL annealing can accelerate convergence.
\end{abstract}

\section{INTRODUCTION}\label{sec:introduction}
Deep latent variable models are generative models that convert latent variables generated from a prior distribution into samples that closely resemble data through a neural network. Variational autoencoders (VAEs) \citep{kingma2013auto, rezende2014stochastic}, one of the deep latent variable models, have been applied in various fields such as image generation \citep{child2020very, vahdat2020nvae}, text generation \citep{bowman2015generating}, music generation \citep{roberts2018hierarchical}, clustering \citep{jiang2016variational}, dimensionality reduction \citep{akkari2022bayesian}, data augmentation \citep{norouzi2020exemplar}, and anomaly detection \citep{an2015variational, park2022interpreting}. 
The objective function of the VAE can be decomposed into the reconstruction error (\textit{distortion}) and KL divergence term (\textit{rate}), which have different roles and a trade-off relationship.
In practice, VAEs are generally trained with the $\beta$-VAE objective \citep{higgins2016beta}, which balances the reconstruction error and KL divergence term by introducing a weight parameter $\beta$.

In addition to data generation tasks, $\beta$-VAEs are state-of-the-art models for representation learning. 
In particular, $\beta$-VAEs have gained attention owing to their capability for obtaining representations in which a single latent variable is sensitive to changes in a single generative factor and is relatively invariant to changes in other factors \citep{bengio2013representation}.
This property of representations is called ``disentanglement''. 
For example, a disentangled representation of 3D objects is sensitive to a single independent data-generative factor, such as object identity, position, scale, and color. In $\beta$-VAE, the degree of disentanglement can be controlled by tuning the weight $\beta$. 
However, this $\beta$-tuning causes a notorious problem in which the variational posterior $q_{\phi}(\B{z}|\B{x})$ tends to align with the prior $p(\B{z})$ during learning, thereby hindering the quality of representation learning. 
This phenomenon is commonly referred to as ``posterior collapse''. 

Although several studies have theoretically analyzed the relationship between $\beta$ turning, disentanglement, and posterior collapse, the understanding remains limited.
In particular, the learning dynamics of $\beta$-VAEs have not been fully explored thus far.
On the other hand, several attempts have been made to mitigate the posterior collapse \citep{yang2017improved, dieng2019avoiding, zhao2017learning, kim2018semi}. 
Among these, the simplest strategy is monotonic KL annealing, in which the weight $\beta$ is scheduled to gradually increase during training \citep{bowman2015generating}.
Although this heuristic method is recognized for its effectiveness, it cannot be guaranteed owing to its limited theoretical understanding.

This study theoretically analyzes a minimal model known as a linear VAE \citep{lucas2019don}, which captures the essence of $\beta$-VAEs. Our results elucidate the formation process of disentangled features, the relationship between $\beta$ and the posterior collapse, and the effect of superfluous latent variables on the generative factors. 
In addition, we reveal the influence of KL annealing on the generalization performance.

\paragraph{Contributions}
This study develops a theory of learning dynamics for VAEs. Specifically, this study rigorously proved that the one-pass gradient descent dynamics (SGD) converges to a deterministic process characterized by ordinary differential equations (ODEs) within the limit of large input dimensions, thereby providing the asymptotically exact dynamics of the generalization error. Consequently, the relationships between the generalization error, the posterior collapse, the disentanglement, and $\beta$ are revealed in two scenarios: the ``model-matched case'' wherein the number of generative factors in the training data matches the dimension of the latent space, and the ``model-mismatched case'' wherein the latent dimension exceeds the number of the generative factors. The main contributions of this study are as follows. 
\begin{itemize}
    \item An asymptotic exact analysis of the macroscopic dynamics by the one-pass SGD is derived. The results demonstrate that the macroscopic dynamics converge to a deterministic process characterized by ODEs within the limit of large input dimensions. 
    \item The stability analysis of the fixed points of the limiting ODEs indicates that when $\beta$ exceeds a certain threshold, posterior collapse is inevitable regardless of the learning time.
    \item Theoretical analysis of the well-known replica method in statistical mechanics and theoretical analysis of the dynamics of SGD are shown to have a complementary relationship. 
    Specifically, a steady state of the SGD dynamics coincides exactly with the global optimum derived by the replica method, indicating the reachability to the global optimum using SGD.
    \item The numerical integration of the ODEs uncovers a phenomenon, wherein the VAE initially learns entangled representations and gradually acquires those that are disentangled. The stability of fixed points indicates that disentangled representations can be achieved for any $\beta$.
    \item The analysis of the model-mismatched case demonstrates that the superfluous latent variable overfits the background noise with a small $\beta$, degrading generalization.
    The stability of the fixed points reveals that while an optimal generalization is achieved for the same $\beta$ value in both the model-matched and model-mismatched cases, the convergence time for the model-mismatched case is significantly longer.
    \item 
    Appropriately tuned KL annealing accelerates the convergence of learning. 
    Additionally, the stability analysis provides a specific annealing rate beyond which the convergence decelerates.
\end{itemize}

\subsection{Preliminaries}\label{subsec:Preliminaries}
Here, we summarize the notations used in this study. The expression $\|\cdot \|_{F}$ denotes the Frobenius norm. 
$I_{N} \in \mab{R}^{N \times N}$ denotes an $N \times N$ identity matrix, whereas $\B{0}_{N}$ denotes the vector $(0, \ldots, 0)^{\top} \in \mab{R}^{N}$. $D_{\mathrm{KL}}[\cdot \| \cdot]$ denotes the Kullback–Leibler (KL) divergence. 

\section{BACKGROUND}\label{sec:dynamics-vae}
\subsection{Variational Autoencoders}
The VAE \citep{kingma2013auto} is a latent generative model. 
Let $\mathcal{D} = \{\B{x}^{\mu}\}_{\mu=1}^{P}$ with $\B{x}^{\mu} \in \mathbb{R}^{D}$ be the training data, and $p_{\mathcal{D}}(\B{x})$ indicate the empirical distribution of the training dataset. 
In practical applications, VAEs are typically trained using the $\beta$-VAE objective \citep{higgins2016beta} defined by 
\begin{multline}
\label{eq:beta-elbo}
    \mathbb{E}_{p_{\mathcal{D}}} \left[\mathbb{E}_{q_{\phi}}[-\log p_{\theta}(\B{x}|\B{z})] + \beta D_{\mathrm{KL}}[q_{\phi}(\B{z}|\B{x}) \| p(\B{z})]  \right] \\
    \delequal \mathbb{E}_{p_{\mathcal{D}}}[l(\theta, \phi;\B{x}, \beta)], 
\end{multline}
where $p(\B{z})$ is a prior for the latent variables, and the parameter $\beta \ge 0$ is introduced to control the trade-off between the first and second terms in Eq.~\eqref{eq:beta-elbo}.
Distributions $p_{\theta}(\B{x}|\B{z})$ characterized by parameters $\theta$ and $q_{\phi}(\B{z}|\B{x})$ by  $\phi$ are commonly referred to as the \textit{decoder} and \textit{encoder}, respectively. 
Subsequently, VAEs optimize both the encoder parameters $\phi$ and decoder parameters $\theta$ by minimizing the objective of Eq.~\eqref{eq:beta-elbo}. 
Note that when $\beta=0$, the objective becomes a deterministic autoencoder that focuses more on minimizing the first term, which is referred to as the \textit{reconstruction error}.

\subsection{Posterior Collapse and KL Annealing}\label{subeq:pc-klannealing}
A notorious problem in VAE optimization is that the variational posterior $q_{\phi}(\B{z}|\B{x})$ frequently aligns closely with the prior $p(\B{z})$, a phenomenon which is known as posterior collapse, hindering the quality of representation learning.
Several attempts have been made to mitigate this problem \citep{yang2017improved, dieng2019avoiding, zhao2017learning, kim2018semi}, among which a simple remedy called monotonic KL annealing has been proposed in \citep{bowman2015generating}, where $\beta=0$ is set at the beginning of the training and gradually increases until $\beta=1$ is reached. In practice, $\beta$ is defined as follows:
\begin{equation}
\label{eq:linear-kl-annealing}
    \beta^{t+1} \gets \beta^{t} + \varepsilon 
\end{equation}
where $t$ denotes each step of the parameter updates using an optimization algorithm, and $\varepsilon$ represents the annealing rate. Monotonic annealing has become a standard method for training VAEs, particularly in numerous natural language processing applications. Although this heuristic is simple and often effective, it is not theoretically guaranteed.
Additionally, cyclical KL Annealing \citep{fu2019cyclical} was utilized, which repeatedly applies monotonic KL annealing in a cyclical manner.

\section{SETTING}\label{subsec:setting}
\paragraph{Generative Model for Real Data} 
We consider that the real dataset $\mathcal{D}=\{\B{x}^{\mu}\}$ with $\mu=1, \ldots, P$, drawn according to the generative model given by the following: 
\begin{equation}
    \B{x}^{\mu} = \sqrt{\frac{\rho}{N}} W^{\ast} \bm{c}^{\mu} + \sqrt{\eta} \B{n}^{\mu}, 
\end{equation}
where $W^{\ast} \in \mathbb{R}^{N \times M^{\ast}}$ is a deterministic unknown feature matrix with $M^{\ast}$ features, $\B{c}^{\mu} \in \mathbb{R}^{M^{\ast}}$ is a random vector drawn from a standard normal distribution $\mac{N}(\B{0}_{M}, I_{M})$, $\B{n}^{\mu}$ is a background noise vector whose components are i.i.d from the standard normal distribution $\mathcal{N}(\B{0}_{N}, I_{N})$, and $\eta \in \mathbb{R}$ and $\rho \in \mab{R}$ are the scalar parameters that control the strength of the noise and signal, respectively. 
This generative model is known as the spiked covariance model \citep{johnstone2009consistency} and is used in the theoretical studies of the principal component analysis (PCA). 
Despite $W^{\ast}$ not being orthogonal, $W^{\ast} \B{c}^{\mu}$ can be rewritten as $(W^{\ast} R) (R^{-1} \B{c})$, where $R$ is a matrix that orthogonalizes and normalizes the columns of $W^{\ast}$. 
This can be considered as an equivalent system in which the new feature vector is $R^{-1} \B{c}$. Therefore, we assume, 
without the loss of generality, we assume that $(W^{\ast})^{\top} W^{\ast} = I_{M}$.

\textbf{Linear VAE Model} 
The linear VAE model \citep{dai2018connections, lucas2019don, sicks2021generalised} consists of a linear decoder and encoder given by
\begin{align}
    &p_{W}(\B{x}|\B{z}) = \mathcal{N}\left(\B{x}; \frac{1}{\sqrt{N}}W\B{z}, I_{N} \right), \\
    &q_{V, D}(\B{z}|\B{x}) = \mathcal{N}\left(\B{z}; \frac{1}{\sqrt{N}} V^{\top} \B{x}, D \right), \\
    &p(\B{z}) = \mathcal{N}(\B{z}; \B{0}_{N}, I_{N}), 
\end{align}
where the diagonal covariance matrix $D \in \mathbb{R}^{M \times M}$  indicates the learning parameters, and $W \in \mathbb{R}^{N \times M}$ and $V \in \mathbb{R}^{N \times M}$ also indicate the learning parameters.  
We assume a fixed identity covariance matrix $I_{N}$ because it is often used in practice. 

\paragraph{Training Algorithm} 
The VAE is trained to learn the generative model using the following optimization problem:
\begin{multline}
\label{eq:linear-vae-loss}
    (\bar{W}(\mathcal{D}), \bar{V}(\mathcal{D}), \bar{D}(\mathcal{D})) \\
    = \underset{W, V, D}{\mathrm{argmin}}~ \mathcal{R}(W, V, D; \mathcal{D}, \beta, \lambda),
\end{multline}
where
\begin{multline}
\label{eq:linear-vea-batch-loss}
    \mathcal{R}(W, V, D; \mathcal{D}, \beta, \lambda) \delequal 
    \sum_{\mu=1}^{P} l(W, V, D; \B{x}^{\mu}, \beta) \\
    + \frac{\lambda}{2} \|W\|_{F}^{2} + \frac{\lambda}{2} \|V\|^{2}_{F}.
\end{multline}
Here, $l(W, V, D; \B{x}, \beta)$ is defined by Eq.~(\ref{eq:beta-elbo}), and the last two terms regulate the magnitudes of the parameters $W$ and $V$, where $\lambda>0$ is a regularization parameter.
We consider a standard training algorithm using the stochastic gradient descent to solve the optimization problem defined in  Eq.~\eqref{eq:linear-vae-loss}. 
To simplify the theoretical analysis, we assume a one-pass setting, where each data sample $\B{x}^{\mu}$ is used only once. 
At $t$ steps, the model parameters $W^{t}$, $V^{t}$ and $D^{t}$ are updated using a new sample $\B{x}^{t}$ according to the following:
\begin{align}
    &W^{t+1} = W^{t} - \tau_{W} \nabla_{W^{t}} r(W^{t}, V^{t}, D^{t}; \beta, \lambda,  \B{x}^{t}), \label{eq:update-W}\\
    &V^{t+1} = V^{t} - \tau_{V} \nabla_{V^{t}} r(W^{t}, V^{t}, D^{t}; \beta, \lambda, \B{x}^{t}), \label{eq:update-tW}\\
    &D^{t+1} = D^{t} - \tau_{D} \nabla_{D^{t}} r(W^{t}, V^{t}, D^{t}; \beta, \lambda, \B{x}^{t})/N, \label{eq:update-V}
\end{align}
where $r$ represents the loss for a given sample defined as follows: 
\begin{multline*}
    r(W^{t}, V^{t}, D^{t}; \beta, \lambda, \B{x}^{t}) \delequal l(W^{t}, V^{t}, D^{t}; \B{x}^{t}, \beta) \\
    + \frac{\lambda}{2N} \|W^{t}\|_{F}^{2} + \frac{\lambda}{2N} \|V^{t}\|^{2}_{F}.
\end{multline*}
Parameters $\tau_{W}$, $\tau_{V}$ and $\tau_{D}$ in the expressions above are the learning rates.  
The SGD algorithm characterizes a Markov process $X^{t} \delequal [W^{t}, \tilde{W}^{t}, V^{t}]$ with an updated rule. Hereafter, $X^{t}$ is referred to as the microscopic state. 
Note that the analysis presented in this study can be naturally extended to the mini-batch SGD where the mini-batch size remains a finite number, that is, $\mac{O}(N^{0})$.

\paragraph{Generalization Metric} 
The VAE can generate a sample $\B{x} \sim p_{W}(\B{x})$ through the following procedure. First, a latent variable $\B{z} \sim p(\B{z})$ is generated followed by a sample $\B{x} \sim p_{W}(\B{x}|\B{z})$. 
Thus, the generalization error $\varepsilon_{g}$ measures the extent of the signal recovery from the training data, defined as follows: 
\begin{equation}
    \varepsilon_{g}(W, W^{\ast}) = \frac{1}{N} \mathbb{E}_{\B{c}} \left[ \left\| \sqrt{\rho} W^{\ast} \B{c} - W \B{c} \right\|^{2} \right],
\end{equation}
where $\mathbb{E}_{\B{c}}[\cdot]$ is the average over $p(\B{c}) = \mathcal{N}(0_{M}, I_{M})$. 

\section{MACROSCOPIC DYNAMICS OF VAE}\label{sec:macro-dynamics}
From a statistical physics perspective, $\varepsilon_{g}(W, W^{\ast})$ can be expressed as a function of the following set of macroscopic variables, called \textit{order parameters}. Based on this idea, we attempt to express the dynamics of $\varepsilon_{g}(W^{t}, W^{\ast})$ by explicitly using the time evolution of the order parameters. 
\begin{Def}
For $X^{t} = [W^{t}, V^{t}, D^{t}]$, the macroscopic variables are defined as follows:
\begin{align*}
    &m^{t} = \frac{1}{N} (W^{t})^{\top} W^{\ast} ,d^{t} = \frac{1}{N} (V^{t})^{\top} W^{\ast}, \\
    &Q^{t} = \frac{1}{N} (W^{t})^{\top} W^{t},E^{t} = \frac{1}{N} (V^{t})^{\top} V^{t}, R^{t} = \frac{1}{N} (W^{t})^{\top} V^{t}.
\end{align*}
Subsequently, to compactly represent the macroscopic variables, the macroscopic state $\mac{M}^{t}$ 
 of the Markov chain in $X^{t}$ is defined as follows:
\begin{equation*}
    \mathcal{M}^{t} \delequal (m^{t}, d^{t}, Q^{t}, E^{t}, R^{t}, V^{t}, D^{t}) \in \mathbb{R}^{M \times (2M^{\ast}+5M)}.
\end{equation*}
\end{Def}
Intuitively, the overlaps $m_{ij}^{t}$ and $ d_{ij}^{t}$ measure the similarity to the $j$-th representation of the true model, i.e., the $j$-th column of $W^{\ast}$; the overlaps $Q_{ij}^{t}$, $E_{ij}^{t}$, and $ R_{ij}^{t}$ measure the similarities between the decoder weights, specifically the $i$-th and $j$-th columns of $W^{t}$, the encoder weights, i.e., the $i$-th and $j$-th columns of $V^{t}$, and between the decoder and encoder weights, i.e., the $i$-th column of $W^{t}$ and the $j$-th column of $V^{t}$, respectively. 
The off-diagonal elements of $E^{t}$ represent the independence of the encoded representations. Thus, if the off-diagonal elements of $E^{t}$ are zero, a disentangled representation is obtained; otherwise, an entangled representation is obtained.

We investigate the dynamics of the training algorithm expressed by Eq.~(\ref{eq:update-W})-(\ref{eq:update-V}) for the macroscopic variables. Our first contribution is to provide rigorous theoretical results under the following assumptions:
\begin{itemize}
    \item [(A.1)] The sequences $\B{c}^{t}$ and $\B{n}^{t}$ for $t=1, \ldots,$ are i.i.d. random variables, and $\B{c}^{t}$ is drawn from the standard normal distribution $\mac{N}(0_{M}, I_{M})$.
    \item [(A.2)] The sequence $\B{n}^{t}$ is drawn from the standard normal distribution $\mac{N}(0_{N}, I_{N})$, and $\{\B{n}^{t}\}$ is independent of $\{\B{c}^{t}\}$. 
    \item [(A.3)] The initial macroscopic state $\mathcal{M}^{0}$ satisfies $\mab{E}\|\mac{M}^{0}-\bar{\mac{M}}^{0} \|_{F} \le C/\sqrt{N}$, where $\bar{\mac{M}}^{0}$ is a deterministic matrix and $C$ is a constant independent of $N$.
    \item [(A.4)] For $i=1, 2, \ldots, N$, the initial microscopic state $X^{0}=[W^{0}, V^{0}, D^{0}]$ satisfies $\mab{E}[\sum_{m=1}^{M}\{(W_{im}^{0})^{4}+(V_{im}^{0})^{4}+(D_{m}^{0})^{4}\} + \sum_{m=1}^{M^{\ast}} (W_{im}^{\ast})^{4}] \le C$, where $C$ is a constant independent of $N$ and $D^{0} \neq \B{0}_{M \times M}$.
\end{itemize}
Assumptions (A.1) and (A.2) for $\B{c}^{t}$ and $\B{n}^{t}$ can be relaxed to non-Gaussian cases if all moments $\B{n}_{t}$ are bounded; however, we use the Gaussian assumption to simplify the proof. 
Assumption (A.3) ensures that the initial macroscopic states converge to deterministic values as the input dimension $N$ approaches infinity.
Assumption (A.4) requires that the elements in the feature matrix $W^{\ast}$ and initial microscopic state $X^{0}$ are $\mac{O}(1)$.
The following theorem proves that the stochastic process of the macroscopic states converges to a deterministic process in the $N \to \infty$ limit characterized by ODEs. 
\begin{Thm}\label{thm:convergence-SGD}
For all $T >0$, it holds under assumptions (A.1)-(A.4) that
\begin{equation}
    \max_{0 \le \mu \le NT} \mab{E} \|\mac{M}^{t} - \mac{M}(t/N) \|_{F} \le \frac{C(T)}{\sqrt{N}},
\end{equation}
where $C(T)$ is a constant that depends on $T$ but not on $N$, and $\mac{M}(t)$ is a unique solution of the ODE
\begin{equation}
\label{eq:ode-macro}
    \frac{\mathrm{d} \mac{M}(t)}{\mathrm{d} t} = F(\mac{M}(t)),
\end{equation}
with the initial condition $\mac{M}(0)=\bar{\mac{M}}^{0}$ and $F: \mab{R}^{M\times (2M^{\ast} + 4M)}$ is uniformly Lipschitz continuous in $\mac{M}(t)$. A specific expression is not demonstrated owing to its length; however, the entire function is provided in Supplementary Materials~\ref{sec:complete-form-F}.
\end{Thm}
The convergence theory of stochastic processes and a coupling trick \citep{wang2018subspace} can prove the theorem. 
To prove this, decompose $\mac{M}^{t}$ into the following: 
\begin{equation*}
    \mac{M}^{t+1}- \mac{M}^{t} = \mab{E}_{t} \mac{M}^{t+1} - \mac{M}^{t} + \left(\mac{M}^{t+1} - \mab{E}_{t} \mac{M}^{t+1} \right) 
\end{equation*}
where $\mab{E}_{t}$ denotes the conditional expectation given the state of the Markov chain $X^{t}$. Thus, it is sufficient to show that the following two conditions hold for all $t \le NT$: 
\begin{align*}
    &\mab{E} \|\mab{E}_{t} \mac{M}^{t+1} - \mac{M}^{t} - F(\mac{M}^{t})/N \|_{F} \le C(T) N^{-2/3} \\
    &\mab{E} \|\mac{M}^{t+1}-\mab{E}_{t} \mac{M}^{t+1}\|_{F}^{2} \le C(T) N^{-2}
\end{align*}
The first condition ensures that the leading order of the average increment is captured by the ODEs in the Theorem \ref{thm:convergence-SGD}. The second condition guarantees that the stochastic part can be ignored in the large $N$ limit. Further details regarding the derivation of these two conditions and the proof of the Theorem \ref{thm:convergence-SGD} can be found in Supplementary Materials~\ref{sec:proof-theorem}.

This theorem indicates that the macroscopic stochastic process $\mac{M}^{t}$ converges to the deterministic process $\mac{M}(t)$ at a convergence rate of $\mac{O}(1/\sqrt{N})$. 
Furthermore, the generalization error $\varepsilon_{g}$ can be expressed as a function of the macroscopic state, which allows us to investigate the dynamics from the ODEs in Eq.~\eqref{eq:ode-macro}. In the following section, we present the results obtained by using Eq.~\eqref{eq:ode-macro}.

\section{RESULTS}\label{sec:results}
We investigate the learning dynamics of VAE with a high-dimensional data limit using Eq.~\eqref{eq:ode-macro}.
Specifically, we focus on the following representative cases: (i) the model-matched setting ($M=M^{\ast}=1$) where the number of generative factors in the generative model, i.e., the number of columns in $W^{\ast}$, is equal to the latent space dimension; and (ii) the model-mismatched setting ($M=2$ and $M^{\ast}=1$), where the latent space dimension is larger than the number of the generative factors. In addition, numerical experiments are conducted to verify the consistency of our theory and to compare the results obtained by training the VAE.

\subsection{Dynamics of Generalization Error}

\begin{figure*}[tb]
    \centering
    \includegraphics[width=\textwidth]{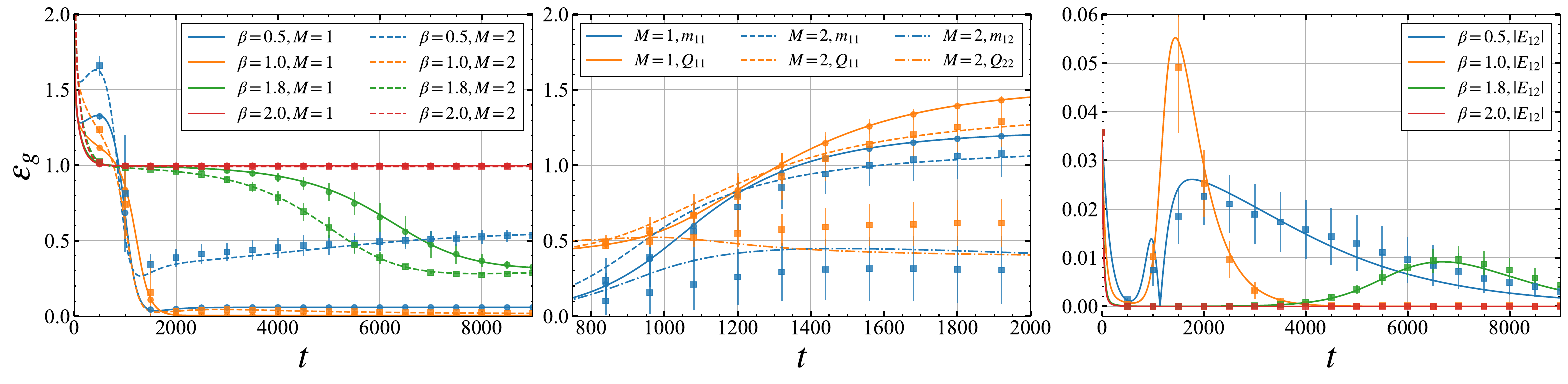}
    \caption{(Left) Generalization error, (middle) order parameters $m$ and $Q$, and (right) order parameter $E_{12}$ as a function of time $t$ for varying $\beta$ values with fixed parameters $\lambda=0, \tau_{W}=\tau_{V}=\tau_{D}=0.01$, and $\rho=\eta=1$ for both model-matched and model-mismatched cases. Each point on the plots represents the averages of five different numerical simulations with $N=500$, and the error bars represent the standard deviations of the results.}
    \label{fig:eg-dynamics}
\end{figure*}

The $\beta$ dependence of learning dynamics is discussed by observing the time evolution of the generalization. 
The results are summarized as follows: 

\paragraph{Peak and Long Plateau in $\varepsilon_{g}$}
Fig.~\ref{fig:eg-dynamics} demonstrates the time dependence of the generalization error $\varepsilon_{g}$ for various $\beta$ values along with the numerical experimental results with finite data dimension.
For a smaller $\beta$, the generalization error $\varepsilon_{g}$ peaks in the early stages of learning, which tends to smoothly disappear as $\beta$ increases.
Furthermore, for a larger $\beta$, a long plateau appears in the range of $t$, and the length of this plateau increases as $\beta$ increases. 
When the value of $\beta$ exceeds $2$, the decrease in the generalization error $\varepsilon_{g}$ appears to completely disappear. We will discuss whether this decrease exists in the infinite time in the following section, based on the fixed points of the ODEs.

\paragraph{Overfitting with a Small $\beta$.}
As shown in Fig.~\ref{fig:eg-dynamics},  the generalization error $\varepsilon_{g}$ decreases followed by an increase near $t \approx 1200$ for a small $\beta$, where the difference between order parameters $m_{11}(t)$ and $Q_{11}(t)$ is minimal.
After passing this point, $M(t)$ saturates to a certain value, and $Q(t)$ continues to increase. 
This behavior indicates that while the recovery of the feature vector becomes saturated, the VAE starts to overfit the background noise. 
This suggests that the early stopping method, which stops the SGD update when the generalization error begins to increase, is effective for small $\beta$.

\paragraph{Formation Process of Disentanglement}
As discussed in Sec.~\ref{sec:macro-dynamics}, the off-diagonal terms of the order parameter $E$ can be used to measure the disentanglement of the obtained representation.  
When these off-diagonal terms $E_{ij}$ are zero, the corresponding representations $z_{i}, z_{j} \sim q_{V, D}(\B{z}|\B{x})$ are disentangled. Conversely, when $E_{ij}\neq 0$, the corresponding representations are entangled.
The right panel of Fig.~\ref{fig:eg-dynamics} shows the time dependence of the off-diagonal term, meaning the formation process of a disentangled representation. 
The representation is entangled, i.e., $E_{12} \neq 0$, in the early stages of learning, and a peak then appears at some time $t$. Subsequently, the representations gradually become disentangled as time progresses; that is, $E_{12} = 0$. The stability of the fixed points determines whether the disentanglement representations are obtained for any $\beta$ in the limit $t \to \infty$. 

\subsection{Steady State of Generalization Error}

\begin{figure}[tb]
    \centering
    \includegraphics[width=\columnwidth]{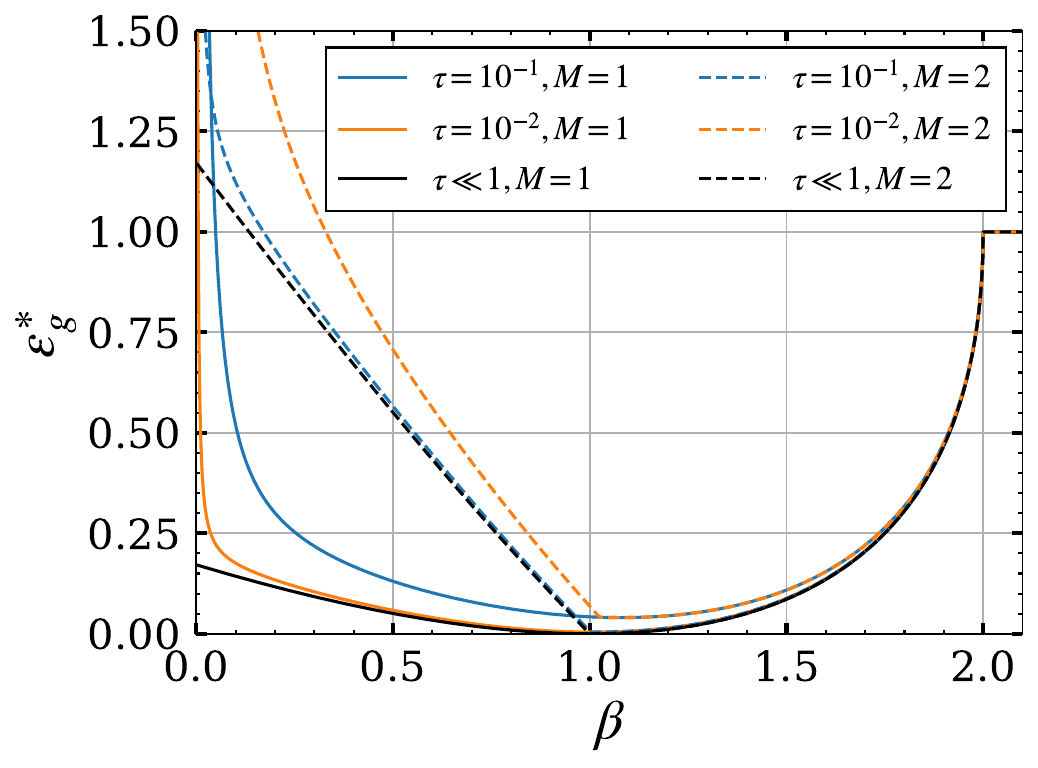}
    \caption{Asymptotic generalization error as a function of time $t$ for the varying learning rate $\beta$ with fixed parameters $\lambda=0$, $\rho=\eta=1$ for both model-matched (solid line) and model-mismatched cases (dashed line).}
    \label{fig:eg-steady}
\end{figure}

Considering the analysis of the dynamics in the previous section, it remains unclear whether it is possible to escape from the plateau and reduce the generalization error $\varepsilon_{g}$ for any given $\beta$, or to obtain disentangled features in the long-time limit. 
In this section, we discuss these issues using a local stability analysis of the ODEs in Eq.~\eqref{eq:ode-macro}.
To further reduce the degrees of freedom of the ODEs, we assume that the regularization parameter $\lambda=0$ and a common learning rate $\tau = \tau_{W}=\eta_{V}=\tau_{D}$. 
In the subsequent analysis, if the Jacobian matrix of the ODEs has only negative eigenvalues, the fixed point is called locally stable, and if the Jacobian matrix has both zero eigenvalues and negative eigenvalues, the fixed point is called marginally stable.

\paragraph{Stability of Model-Matched Case}
We investigate the local stability of the fixed points of the ODEs in the model-matched case to clarify the conditions under which the VAE encounters a posterior collapse.
\begin{Thm}\label{theorem:steady-match}
    For a small learning rate $\tau$ limit and $\lambda=0$, the fixed points of ODEs in the model-matched case with $M=M^{\ast}=1$ have the following properties.
    \begin{itemize}
        \item For $\beta < \rho + \eta$, the following fixed point is locally stable:
        \begin{align}
            &m^{\ast} = \sqrt{\rho+\eta - \beta}, \\ 
            &\varepsilon_{g}^{\ast} = \rho - \sqrt{\eta + \rho - \beta} (2 \sqrt{\rho} - \sqrt{\eta + \rho - \beta}), \label{eq:1rank-asymp-error}
        \end{align}
        \item For $\beta = \rho + \eta$, the fixed point, $m^{\ast} = 0,~\varepsilon_{g}^{\ast} = \rho$, is marginally stable.
        \item For $\beta > \rho + \eta$, the fixed point,$m^{\ast} = 0,~\varepsilon_{g}^{\ast} = \rho$, is locally stable.
    \end{itemize}
\end{Thm}
Theorem~\ref{theorem:steady-match} elucidates that once $\beta$ exceeds the threshold $\beta^{\ast} = \rho + \eta$, the generalization error can not escape from the plateau, despite $t$ increasing, which indicates that the posterior collapse cannot be avoided.

Furthermore, the limiting value of the generalization error $\varepsilon_{g}^{\ast}$ coincides with that obtained from the analysis of the global optimum of Eq.~\eqref{eq:linear-vea-batch-loss} \citep{ichikawa2022statistical}; namely, following Remark holds.
\begin{Rem}
The limiting value of the generalization error in Eq.~\eqref{eq:1rank-asymp-error} exactly equals the generalization error derived in the infinite data size limit by the analysis of the global optimum using the replica method \citep{ichikawa2023dataset}.
\end{Rem}
This result implies that it is possible to reach a global optimum solution using SGD with a small learning rate limit.
To our best knowledge, the exact correspondence between the global optima obtained using the replica method and the steady state of the one-pass SGD and the reachability to the global optima has not yet been explored in the statistical physics community.

\paragraph{Stability of Model Mismatched-Case}
We also clarify the condition under which the VAE encounters a posterior collapse in the model-mismatched case and obtains disentangled representations.
\begin{Thm}\label{theorem:steady-mismatch}
    For a small learning rate $\tau$ limit and $\lambda=0$, the fixed points of ODEs in the model mismatched case with $M=2$ and $M^\ast=1$ have the following properties. 
    \begin{itemize}
        \item For $\beta<\eta$, the following fixed point is locally stable: 
        \begin{align*}
            &m^{\ast} = (\sqrt{\rho + \eta - \beta}, 0),~(0, \sqrt{\rho + \eta - \beta}),~E^{\ast}_{12}=0 \\
            &Q^{\ast} = 
            \begin{pmatrix}
                \rho + \eta - \beta & 0 \\
                0 & \eta -\beta
            \end{pmatrix},~ 
            \begin{pmatrix}
                \eta -\beta & 0 \\
                0 & \rho + \eta - \beta
            \end{pmatrix} \\
            &\varepsilon_{g}^{\ast} = \rho - \sqrt{\eta + \rho - \beta} (2 \sqrt{\rho} - \sqrt{\eta + \rho - \beta}) + \eta -\beta, 
        \end{align*} 
        \item For $\beta=\eta$, the fixed point is marginally stable:
        \begin{align*}
            &m^{\ast} = (\sqrt{\rho}, 0),~(0, \sqrt{\rho}),~E^{\ast}_{12}=0 \\
            &Q^{\ast} = 
            \begin{pmatrix}
                \eta & 0 \\
                0 & 0
            \end{pmatrix},~ 
            \begin{pmatrix}
                0 & 0 \\
                0 & \eta
            \end{pmatrix}, ~~\varepsilon_{g}^{\ast} = 0.
        \end{align*}
        \item For $\eta < \beta < \rho + \eta$, the fixed point is locally stable:
        \begin{align*}
            &m^{\ast} = (\sqrt{\rho + \eta - \beta}, 0),~(0, \sqrt{\rho + \eta - \beta}),~E^{\ast}_{12}=0 \\
            &Q^{\ast} = 
            \begin{pmatrix}
                \rho + \eta - \beta & 0 \\
                0 & 0
            \end{pmatrix},~ 
            \begin{pmatrix}
                0 & 0 \\
                0 & \rho + \eta - \beta
            \end{pmatrix} \\
            &\varepsilon_{g}^{\ast} = \rho - \sqrt{\eta + \rho - \beta} (2 \sqrt{\rho} - \sqrt{\eta + \rho - \beta}).
        \end{align*}
        \item For $\beta=\rho+\eta$, the fixed point, $m^{\ast}=Q^{\ast}=\B{0}_{2\times 2},~ E_{12}^{\ast}=0, \varepsilon_{g}^{\ast}=\rho$, is marginally stable.
        \item For $\beta > \rho + \eta$, the fixed point, $m^{\ast}=Q^{\ast}=\B{0}_{2\times 2}, E_{12}^{\ast}=0, \varepsilon_{g}^{\ast}=\rho$, is locally stable.
    \end{itemize}
 \end{Thm}
This theorem indicates that disentangled representations can be obtained in the small learning rate limit for any $\beta$, that is, $\forall \beta, E_{12}^{\ast}=0$. 
The threshold for the posterior collapse is the same as that of the model-matched case.
Thus, Theorem~\ref{theorem:steady-match} and \ref{theorem:steady-mismatch} suggest that $\beta$ can be a risky parameter since the posterior collapse is inevitable regardless of the training period. 
Furthermore, the extremum calculations of the generalization error in Theorems~\ref{theorem:steady-match} and \ref{theorem:steady-mismatch} demonstrate that the generalization error is minimized when $ \beta = \eta$, which means that the best generalization is achieved when $\beta$ is equal to the strength of the background noise $\eta$.
Note that the generalization error in the model-mismatched case at $\beta=\eta$ is marginally stable. 
However, the other fixed points are unstable, indicating that the dynamics converges to the optimal fixed point, but the convergence is significantly slow.

Another noteworthy observation is that Theorem~\ref{theorem:steady-mismatch} shows a new stable fixed point; when $\beta < \eta$, despite $m^{\ast}=(\sqrt{\rho+\eta-\beta}, 0), (0, \sqrt{\rho+\eta-\beta})$ having the same stable fixed point as in the range $\eta < \beta < \rho + \beta$, a non-corresponding element of $Q^{\ast}$ becomes finite,i.e., when $m^{\ast}_{11} \neq 0$, $q^{\ast}_{22} \neq 0$, and when $m^{\ast}_{12} \neq 0$, $q^{\ast}_{11} \neq 0$. 
This suggests that when $\beta < \eta$, the superfluous latent variable for the data-generative factor overfits the background noise and affects the generalization. 
\subsection{Learning Dynamics with KL Annealing}
We now discuss the effectiveness of monotonic KL annealing for the learning dynamics. A stability analysis of the fixed point is conducted for the continuous tanh KL annealing, given by $\beta(t) = \tanh(\gamma t)$, where $\gamma$ denotes the annealing rate. 
This annealing satisfies
\begin{equation}
    \frac{d\beta(t)}{dt} = \gamma (1-\beta^{2}(t)), ~~\beta(0)=0.
\end{equation}
Compared to monotonic KL annealing expressed in Eq.~\ref{eq:linear-kl-annealing}, the trajectories of both tanh KL annealing and monotonic KL annealing are qualitatively similar.
The learning curve with monotonic KL annealing is similar to that with tanh KL annealing; see Supplementary Materials~\ref{subsec:linear-annealing} for the detailed results. In particular, we focus on the representative model-matched case $M=M^{\ast}=1$ with tanh KL annealing. The results are summarized as follows.
\begin{figure}[tb]
    \centering
    \includegraphics[width=\columnwidth]{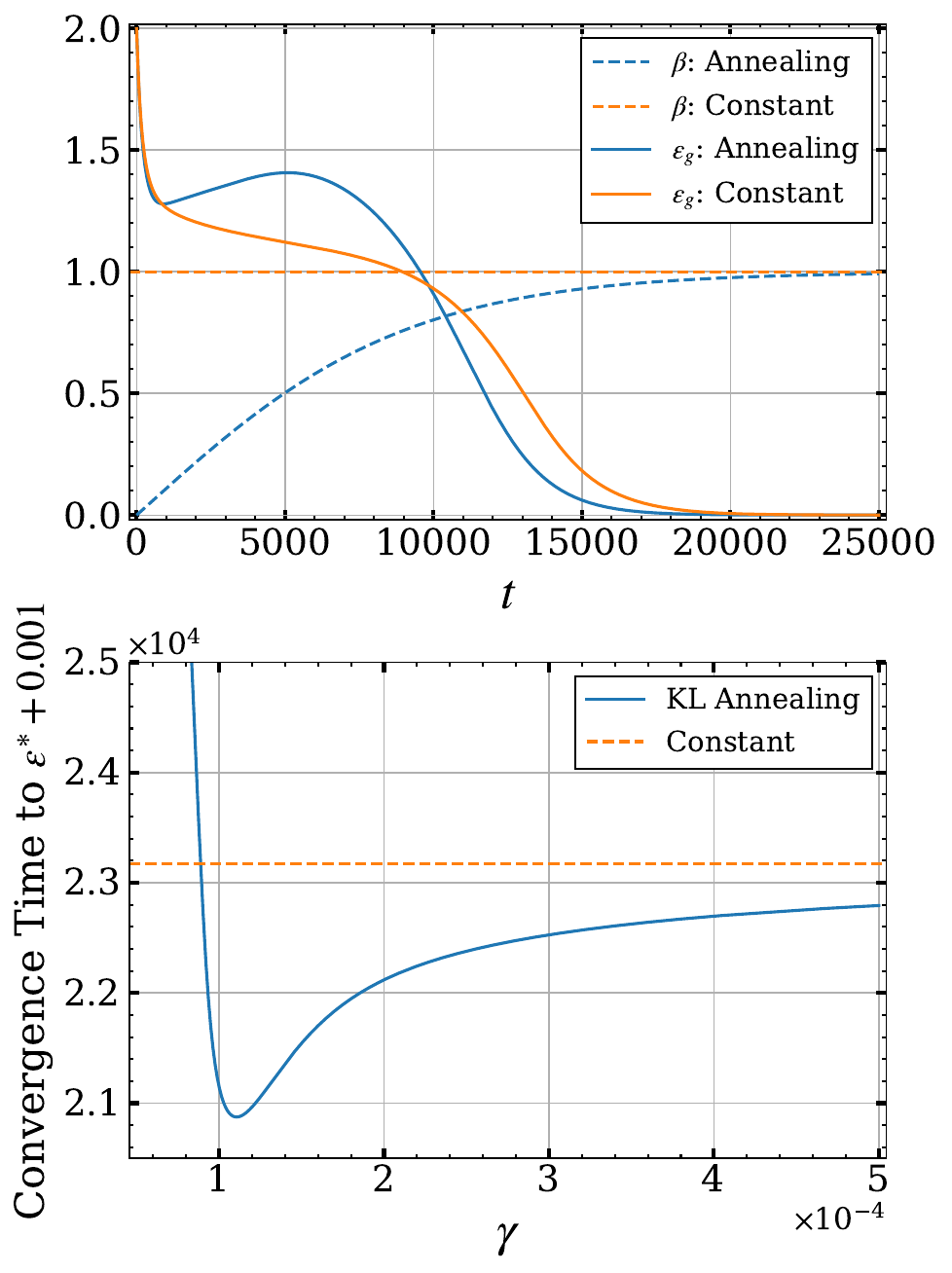}
    \caption{(Top) Time dependence of the generalization error and $\beta$ with both tanh KL annealing for $\beta$ and the constant $\beta=1$ under fixed parameters $\lambda=0$, $\rho=\eta=1$, and $\tau=1$. (Bottom) Annealing-rate $\gamma$ dependence of convergence time to the quasi-steady state deviating by $0.001$, i.e., $\varepsilon^{\ast}_{g}+0.001$. The annealing rate $\gamma$ of the learning dynamics with the tanh KL annealing in the top figure is used as the optimal value obtained from the bottom figure.}
    \label{fig:kl-annealing-1rank}
\end{figure}

\paragraph{Dynamical Properties of KL Annealing}
The top panel of Fig.~\ref{fig:kl-annealing-1rank} demonstrates a comparison of the learning dynamics $\varepsilon_{g}$ with constant $\beta=1$ and tanh KL annealing. 
The bottom panel of Fig.~\ref{fig:kl-annealing-1rank} shows the convergence time to 
 the quasi-steady state $\varepsilon^{\ast}_{g} + 0.001$ as a function of the annealing rate $\gamma$. 
This figure indicates the existence of an optimal annealing rate that maximizes the convergence speed to the quasi-steady state, and that an extremely slow KL annealing rate delays the convergence time. 
The annealing rate $\gamma$ of the learning dynamics using tanh KL annealing, shown in Fig.~\ref{fig:kl-annealing-1rank} (Top), is selected as the optimal rate based on the bottom figure. 
Fig.~\ref{fig:kl-annealing-1rank} demonstrates that the convergence of the generalization error $\varepsilon_{g}$ becomes faster with tanh KL annealing than without it. 
Subsequent discussions will focus on the threshold value of the annealing rate $\gamma$, which adversely affects the learning dynamics.

\paragraph{Steady State with KL Annealing}
Based on the stability analysis of the fixed points, including the time-dependent $\beta(t)$, the learning dynamics using tanh KL annealing exhibit the same stable fixed points. Furthermore, unless excessively slow tanh KL annealing is used, the convergence speed to the steady state coincides with that without the tanh KL annealing. Formally, the following theorem holds:
\begin{Thm}\label{eq:bad-effect-cond-kl-annealing}
    Even when tanh KL annealing is used, its steady state coincides with the steady state of the model-matched case at $\beta = 1$ and $\lambda=0$ without tanh KL annealing. 
    Moreover, when $\rho = 2 - \nu$ and $\eta = \nu$, tanh KL annealing leads to a slow convergence under the condition, $\gamma \le -J_{\max}/2$ where
    \begin{equation*}
        J_{\max} = 
        \begin{cases}
            \frac{\tau}{2}(\sqrt{5}-3),~\frac{\tau (1-2\sqrt{2}+\sqrt{5})}{4} \le \nu \le \frac{\tau (1+2\sqrt{2}+\sqrt{5})}{4} \\
            -\tau (2\nu +1) + \tau \sqrt{4\nu(2\nu-1) +1},~\mathrm{otherwise},
        \end{cases}
    \end{equation*}
    and the convergence using tanh KL annealing becomes the same as that without annealing when $\gamma > -J_{\max}/2$.
\end{Thm}
The proof of this theorem can be found in Supplementary Materials~\ref{subsec:stability-tanh-annealing}. 

\subsection{Related Work}\label{subsec:related-work}

\paragraph{Deterministic Dynamical Descriptions of SGD} 
Deterministic dynamical descriptions of SGD at a high-dimensional input limit have been studied in the statistical physics community. This started with single- and two-layer neural networks with a few hidden units \citep{kinzel1990improving, kinouchi1992optimal, copelli1995line, biehl1995learning, riegler1995line, vicente1998statistical}, based on a heuristic derivation of ODEs describing typical learning dynamics. These results have recently been rigorously proven using the concentration phenomena in stochastic processes \citep{wang2018subspace},  based on which the analysis of the SGD for the two-layer neural networks was proven \citep{goldt2019dynamics, veiga2022phase}. For generative models, the SGD of generative adversarial networks has been investigated \citep{wang2019solvable}. However, to our best knowledge, this analysis has not been applied to the analysis of VAEs thus far.

\paragraph{Linear VAEs}
The linear VAE is a simple model in which both the encoder and decoder are restricted to affine transformations \citep{lucas2019don}. Although deriving analytical results for deep latent models is often intractable, a linear VAE can provide analytical results, facilitating a deeper understanding of VAEs. 
Furthermore, despite this simplicity, the theoretical results can sufficiently explain the behavior of deeper and intricately structured VAEs \citep{lucas2019don, bae2022multi}. 
In fact, results proven to be effective for linear models have been applied to deeper models, leading to the new algorithms \citep{bae2022multi}. 
In addition, several theoretical results have been obtained; \citet{dai2018connections} demonstrated the connections between linear VAE, probabilistic PCA \citep{tipping1999probabilistic}, and robust PCA \citep{candes2011robust, chandrasekaran2011rank}. Simultaneously, studies by \citet{lucas2019don} and \citet{wang2022posterior} used linear VAEs to explore the origins of posterior collapse. 
However, these analyses did not address the learning dynamics indicated in our study.

\section{CONCLUSION}\label{sec:conclusion}
This study rigorously proves that the SGD dynamics of a linear VAE converges to a deterministic process at a high-dimensional input limit. Our analysis reveals that the VAE initially learns entangled representations and then learns disentangled representations. 
Based on the stability analysis, we demonstrate that a posterior collapse occurs at a certain threshold of $\beta$, and superfluous latent spaces can overfit the background noise of training data. 
We also demonstrate that appropriately adjusting KL annealing can accelerate the convergence of training. 
This study has the following limitations. First, our analysis is based on a one-pass SGD, indicating that each data can be used only once; however, this is not the case in practical scenarios. Second, the data generation processes in the real world and VAEs are more complex than those in our data generative model and linear VAE. Thus, a more robust and minimal setup that can overcome these limitations will be developed in the future, along with a novel theoretical method.
\bibliography{ref-vae-dynamics}

\begin{thebibliography}{}

\bibitem[Akkari et~al., 2022]{akkari2022bayesian}
Akkari, N., Casenave, F., Hachem, E., and Ryckelynck, D. (2022).
\newblock A bayesian nonlinear reduced order modeling using variational autoencoders.
\newblock {\em Fluids}, 7(10):334.

\bibitem[An and Cho, 2015]{an2015variational}
An, J. and Cho, S. (2015).
\newblock Variational autoencoder based anomaly detection using reconstruction probability.
\newblock {\em Special lecture on IE}, 2(1):1--18.

\bibitem[Bae et~al., 2022]{bae2022multi}
Bae, J., Zhang, M.~R., Ruan, M., Wang, E., Hasegawa, S., Ba, J., and Grosse, R. (2022).
\newblock Multi-rate vae: Train once, get the full rate-distortion curve.
\newblock {\em arXiv preprint arXiv:2212.03905}.

\bibitem[Bengio et~al., 2013]{bengio2013representation}
Bengio, Y., Courville, A., and Vincent, P. (2013).
\newblock Representation learning: A review and new perspectives.
\newblock {\em IEEE transactions on pattern analysis and machine intelligence}, 35(8):1798--1828.

\bibitem[Biehl and Schwarze, 1995]{biehl1995learning}
Biehl, M. and Schwarze, H. (1995).
\newblock Learning by on-line gradient descent.
\newblock {\em Journal of Physics A: Mathematical and general}, 28(3):643.

\bibitem[Billingsley, 2013]{billingsley2013convergence}
Billingsley, P. (2013).
\newblock {\em Convergence of probability measures}.
\newblock John Wiley \& Sons.

\bibitem[Bowman et~al., 2015]{bowman2015generating}
Bowman, S.~R., Vilnis, L., Vinyals, O., Dai, A.~M., Jozefowicz, R., and Bengio, S. (2015).
\newblock Generating sentences from a continuous space.
\newblock {\em arXiv preprint arXiv:1511.06349}.

\bibitem[Cand{\`e}s et~al., 2011]{candes2011robust}
Cand{\`e}s, E.~J., Li, X., Ma, Y., and Wright, J. (2011).
\newblock Robust principal component analysis?
\newblock {\em Journal of the ACM (JACM)}, 58(3):1--37.

\bibitem[Chandrasekaran et~al., 2011]{chandrasekaran2011rank}
Chandrasekaran, V., Sanghavi, S., Parrilo, P.~A., and Willsky, A.~S. (2011).
\newblock Rank-sparsity incoherence for matrix decomposition.
\newblock {\em SIAM Journal on Optimization}, 21(2):572--596.

\bibitem[Child, 2020]{child2020very}
Child, R. (2020).
\newblock Very deep vaes generalize autoregressive models and can outperform them on images.
\newblock {\em arXiv preprint arXiv:2011.10650}.

\bibitem[Copelli and Caticha, 1995]{copelli1995line}
Copelli, M. and Caticha, N. (1995).
\newblock On-line learning in the committee machine.
\newblock {\em Journal of Physics A: Mathematical and General}, 28(6):1615.

\bibitem[Dai et~al., 2018]{dai2018connections}
Dai, B., Wang, Y., Aston, J., Hua, G., and Wipf, D. (2018).
\newblock Connections with robust pca and the role of emergent sparsity in variational autoencoder models.
\newblock {\em The Journal of Machine Learning Research}, 19(1):1573--1614.

\bibitem[Dieng et~al., 2019]{dieng2019avoiding}
Dieng, A.~B., Kim, Y., Rush, A.~M., and Blei, D.~M. (2019).
\newblock Avoiding latent variable collapse with generative skip models.
\newblock In {\em The 22nd International Conference on Artificial Intelligence and Statistics}, pages 2397--2405. PMLR.

\bibitem[Fu et~al., 2019]{fu2019cyclical}
Fu, H., Li, C., Liu, X., Gao, J., Celikyilmaz, A., and Carin, L. (2019).
\newblock Cyclical annealing schedule: A simple approach to mitigating kl vanishing.
\newblock {\em arXiv preprint arXiv:1903.10145}.

\bibitem[Goldt et~al., 2019]{goldt2019dynamics}
Goldt, S., Advani, M., Saxe, A.~M., Krzakala, F., and Zdeborov{\'a}, L. (2019).
\newblock Dynamics of stochastic gradient descent for two-layer neural networks in the teacher-student setup.
\newblock {\em Advances in neural information processing systems}, 32.

\bibitem[Higgins et~al., 2016]{higgins2016beta}
Higgins, I., Matthey, L., Pal, A., Burgess, C., Glorot, X., Botvinick, M., Mohamed, S., and Lerchner, A. (2016).
\newblock beta-vae: Learning basic visual concepts with a constrained variational framework.
\newblock In {\em International conference on learning representations}.

\bibitem[Ichikawa and Hukushima, 2022]{ichikawa2022statistical}
Ichikawa, Y. and Hukushima, K. (2022).
\newblock Statistical-mechanical study of deep boltzmann machine given weight parameters after training by singular value decomposition.
\newblock {\em Journal of the Physical Society of Japan}, 91(11):114001.

\bibitem[Ichikawa and Hukushima, 2023]{ichikawa2023dataset}
Ichikawa, Y. and Hukushima, K. (2023).
\newblock Dataset size dependence of rate-distortion curve and threshold of posterior collapse in linear vae.
\newblock {\em arXiv preprint arXiv:2309.07663}.

\bibitem[Jiang et~al., 2016]{jiang2016variational}
Jiang, Z., Zheng, Y., Tan, H., Tang, B., and Zhou, H. (2016).
\newblock Variational deep embedding: An unsupervised and generative approach to clustering.
\newblock {\em arXiv preprint arXiv:1611.05148}.

\bibitem[Johnstone and Lu, 2009]{johnstone2009consistency}
Johnstone, I.~M. and Lu, A.~Y. (2009).
\newblock On consistency and sparsity for principal components analysis in high dimensions.
\newblock {\em Journal of the American Statistical Association}, 104(486):682--693.

\bibitem[Kim et~al., 2018]{kim2018semi}
Kim, Y., Wiseman, S., Miller, A., Sontag, D., and Rush, A. (2018).
\newblock Semi-amortized variational autoencoders.
\newblock In {\em International Conference on Machine Learning}, pages 2678--2687. PMLR.

\bibitem[Kingma and Welling, 2013]{kingma2013auto}
Kingma, D.~P. and Welling, M. (2013).
\newblock Auto-encoding variational bayes.
\newblock {\em arXiv preprint arXiv:1312.6114}.

\bibitem[Kinouchi and Caticha, 1992]{kinouchi1992optimal}
Kinouchi, O. and Caticha, N. (1992).
\newblock Optimal generalization in perceptions.
\newblock {\em Journal of Physics A: mathematical and General}, 25(23):6243.

\bibitem[Kinzel and Rujan, 1990]{kinzel1990improving}
Kinzel, W. and Rujan, P. (1990).
\newblock Improving a network generalization ability by selecting examples.
\newblock {\em Europhysics Letters}, 13(5):473.

\bibitem[Kushner, 2009]{Kushner2009}
Kushner, H.~J. (2009).
\newblock {\em Stochastic Approximation and Recursive Algorithms and Applications (Stochastic Modelling and Applied Probability, 35)}.
\newblock Springer New York.

\bibitem[Lucas et~al., 2019]{lucas2019don}
Lucas, J., Tucker, G., Grosse, R.~B., and Norouzi, M. (2019).
\newblock Don't blame the elbo! a linear vae perspective on posterior collapse.
\newblock {\em Advances in Neural Information Processing Systems}, 32.

\bibitem[Norouzi et~al., 2020]{norouzi2020exemplar}
Norouzi, S., Fleet, D.~J., and Norouzi, M. (2020).
\newblock Exemplar vae: Linking generative models, nearest neighbor retrieval, and data augmentation.
\newblock {\em Advances in Neural Information Processing Systems}, 33:8753--8764.

\bibitem[Park et~al., 2022]{park2022interpreting}
Park, S., Adosoglou, G., and Pardalos, P.~M. (2022).
\newblock Interpreting rate-distortion of variational autoencoder and using model uncertainty for anomaly detection.
\newblock {\em Annals of Mathematics and Artificial Intelligence}, pages 1--18.

\bibitem[Rezende et~al., 2014]{rezende2014stochastic}
Rezende, D.~J., Mohamed, S., and Wierstra, D. (2014).
\newblock Stochastic backpropagation and approximate inference in deep generative models.
\newblock In {\em International conference on machine learning}, pages 1278--1286. PMLR.

\bibitem[Riegler and Biehl, 1995]{riegler1995line}
Riegler, P. and Biehl, M. (1995).
\newblock On-line backpropagation in two-layered neural networks.
\newblock {\em Journal of Physics A: Mathematical and General}, 28(20):L507.

\bibitem[Roberts et~al., 2018]{roberts2018hierarchical}
Roberts, A., Engel, J., Raffel, C., Hawthorne, C., and Eck, D. (2018).
\newblock A hierarchical latent vector model for learning long-term structure in music.
\newblock In {\em International conference on machine learning}, pages 4364--4373. PMLR.

\bibitem[Sicks et~al., 2021]{sicks2021generalised}
Sicks, R., Korn, R., and Schwaar, S. (2021).
\newblock A generalised linear model framework for $\beta$-variational autoencoders based on exponential dispersion families.
\newblock {\em The Journal of Machine Learning Research}, 22(1):10539--10579.

\bibitem[Tipping and Bishop, 1999]{tipping1999probabilistic}
Tipping, M.~E. and Bishop, C.~M. (1999).
\newblock Probabilistic principal component analysis.
\newblock {\em Journal of the Royal Statistical Society Series B: Statistical Methodology}, 61(3):611--622.

\bibitem[Vahdat and Kautz, 2020]{vahdat2020nvae}
Vahdat, A. and Kautz, J. (2020).
\newblock Nvae: A deep hierarchical variational autoencoder.
\newblock {\em Advances in neural information processing systems}, 33:19667--19679.

\bibitem[Veiga et~al., 2022]{veiga2022phase}
Veiga, R., Stephan, L., Loureiro, B., Krzakala, F., and Zdeborov{\'a}, L. (2022).
\newblock Phase diagram of stochastic gradient descent in high-dimensional two-layer neural networks.
\newblock {\em arXiv preprint arXiv:2202.00293}.

\bibitem[Vicente et~al., 1998]{vicente1998statistical}
Vicente, R., Kinouchi, O., and Caticha, N. (1998).
\newblock Statistical mechanics of online learning of drifting concepts: A variational approach.
\newblock {\em Machine Learning}, 32:179--201.

\bibitem[Wang et~al., 2018]{wang2018subspace}
Wang, C., Eldar, Y.~C., and Lu, Y.~M. (2018).
\newblock Subspace estimation from incomplete observations: A high-dimensional analysis.
\newblock {\em IEEE Journal of Selected Topics in Signal Processing}, 12(6):1240--1252.

\bibitem[Wang et~al., 2019]{wang2019solvable}
Wang, C., Hu, H., and Lu, Y. (2019).
\newblock A solvable high-dimensional model of gan.
\newblock {\em Advances in Neural Information Processing Systems}, 32.

\bibitem[Wang and Ziyin, 2022]{wang2022posterior}
Wang, Z. and Ziyin, L. (2022).
\newblock Posterior collapse of a linear latent variable model.
\newblock {\em Advances in Neural Information Processing Systems}, 35:37537--37548.

\bibitem[Yang et~al., 2017]{yang2017improved}
Yang, Z., Hu, Z., Salakhutdinov, R., and Berg-Kirkpatrick, T. (2017).
\newblock Improved variational autoencoders for text modeling using dilated convolutions.
\newblock In {\em International conference on machine learning}, pages 3881--3890. PMLR.

\bibitem[Zhao et~al., 2017]{zhao2017learning}
Zhao, T., Zhao, R., and Eskenazi, M. (2017).
\newblock Learning discourse-level diversity for neural dialog models using conditional variational autoencoders.
\newblock {\em arXiv preprint arXiv:1703.10960}.

\end{thebibliography}

\onecolumn
\aistatstitle{Learning Dynamics in Linear VAE: Posterior Collapse Threshold, Superfluous Latent Space Pitfalls, and Speedup with KL Annealing: Supplementary Materials}

\setcounter{section}{0}
\renewcommand{\thesection}{\Alph{section}}


\section{COMPLETE FORM OF THE ORDINARY DIFFERENTIAL EQUATIONS IN THEOREM 4.2}\label{sec:complete-form-F}
In this section, we present the specific function set of $F$ in Theorem 4.2 as follows:
\begin{align}
    \frac{\mathrm{d} m_{ml}}{\mathrm{d} t} &\delequal F_{m_{ml}}(\mac{M}) = - \tau_{W} \Bigg(\sum_{n^{\prime}=1}^{M} m_{n^{\prime}l} h(d_{m}, d_{n^{\prime}}, E_{mn^{\prime}}) +m_{ml} (D_{m}+\lambda)-h(m_{l}^{\ast}, d_{m}, d_{ml})  \Bigg), \label{eq:supp-f-form-m}\\
    \frac{\mathrm{d} d_{ml}}{\mathrm{d} t} &\delequal F_{d_{ml}}(\mac{M})  = - \tau_{V} \left(\sum_{n^{\prime}=1}^{M} Q_{mn^{\prime}} h(m^{\ast}_{l}, d_{n^{\prime}}, d_{n^{\prime}l}) + \beta h(m_{l}^{\ast}, d_{m}, d_{ml})-h(m_{l}^{\ast}, m_{m}, m_{ml}) + \lambda d_{ml}  \right),  \label{eq:supp-f-form-d} \\
    \frac{\mathrm{d} Q_{mn}}{\mathrm{d} t} &\delequal F_{Q_{mn}}(\mac{M}) = -\tau_{W} 
    \Bigg(Q_{mn}(D_{m} + D_{n} + 2\lambda)-h(d_{m}, m_{n}, R_{nm}) - h(d_{n}, m_{m}, R_{mn}) \notag \\
    &+ \sum_{n^{\prime}=1}^{M} Q_{mn^{\prime}} h(d_{n}, d_{n^{\prime}}, E_{nn^{\prime}}) + \sum_{n^{\prime}=1}^{M} Q_{nn^{\prime}} h(d_{m}, d_{n^{\prime}}, E_{mn^{\prime}}) \Bigg) + \eta \tau_{W}^{2} h(d_{m}, d_{n}, E_{mn}), \label{eq:supp-f-form-Q} \\
    \frac{\mathrm{d} E_{mn}}{\mathrm{d} t} &\delequal F_{E_{mn}}(\mac{M}) = - \tau_{V} \Bigg(2 \beta h(d_{m}, d_{n}, E_{mn})-h(m_{m}, d_{n}, R_{mn}) - h(m_{n}, d_{m}, R_{nm}) + 2 \lambda E_{mn}  \notag \\
    &+ \sum_{n^{\prime}=1}^{M} Q_{nn^{\prime}} h(d_{m}, d_{n^{\prime}}, E_{mn^{\prime}}) + \sum_{n^{\prime}=1}^{M} Q_{mn^{\prime}} h(d_{n}, d_{n^{\prime}}, E_{nn^{\prime}}) \Bigg) \notag \\
    &+ \eta \tau_{V}^{2} \Bigg\{\sum_{n^{\prime}, m^{\prime}} Q_{mm^{\prime}} Q_{nn^{\prime}} h(d_{m^{\prime}}, d_{n^{\prime}}, E_{m^{\prime}n^{\prime}}) +\beta \Bigg(\sum_{n^{\prime}=1}^{M} Q_{mn^{\prime}} h(d_{n^{\prime}}, d_{n}, E_{nn^{\prime}}) + \sum_{n^{\prime}=1}^{M} Q_{nn^{\prime}}h(d_{n^{\prime}}, d_{m}, E_{mn^{\prime}})  \notag \\
    &+\beta h(d_{m}, d_{n}, E_{mn}) - h(d_{m}, m_{n}, R_{nm})-h(d_{n}, m_{m}, R_{mn}) \Bigg) \notag \\
    &+ \Bigg(h(m_{m}, m_{n}, Q_{mn})-\sum_{n^{\prime}=1}^{M}Q_{mn^{\prime}} h(d_{n^{\prime}}, m_{n}, R_{nn^{\prime}}) - \sum_{n^{\prime}=1}^{M} Q_{nn^{\prime}}h(d_{n^{\prime}}, m_{m}, R_{mn^{\prime}}) \Bigg) \Bigg\}, \label{eq:supp-f-form-E}\\
    \frac{dR_{mn}}{dt} &\delequal F_{R_{mn}}(\mac{M}) = -\tau_{W} \Bigg(\sum_{n^{\prime}=1}^{M} R_{n^{\prime}n} h(d_{n^{\prime}}, d_{m}, E_{mn^{\prime}}) - h(d_{m}, d_{n}, E_{mn}) + (D_{m}+\lambda)R_{mn} \Bigg)\notag \\
    &-\tau_{V} \Bigg(\sum_{n^{\prime}=1}^{M} Q_{nn^{\prime}}h(m_{m}, d_{n^{\prime}}, R_{mn^{\prime}}) + \beta h(m_{m}, d_{n}, R_{mn}) - h(m_{m}, m_{n}, Q_{mn}) + \lambda R_{mn}  \Bigg) \notag \\
    &+ \tau_{V} \tau_{W} \Bigg(\sum_{n^{\prime}=1}^{M} Q_{nn^{\prime}} h(d_{m}, d_{n^{\prime}}, E_{mn^{\prime}}) + \beta h(d_{m}, d_{n}, E_{mn}) - h(d_{m}, m_{n}, R_{nm}) \Bigg), \label{eq:supp-f-form-R} \\
    \frac{\mathrm{d} D_{m}} {\mathrm{d} t} &\delequal F_{D_{m}}(\mac{M}) =  \tau_{D} \left(\frac{\beta}{D_{m}}-(Q_{mm} + \beta) \right), \label{eq:supp-f-form-D}
\end{align}
where $m^{\ast} = (W^{\ast})^{\top} W^{\ast} /N$, and we use the shorthand expression given by
\begin{equation}
    h(A, B, C) = \rho \sum_{s=1}^{M^{\ast}} A_{s} B_{s} + \eta C.
\end{equation}


\section{PROOF OF THEOREM 4.2}\label{sec:proof-theorem}
In this section, we provide a proof of Theorem 4.2 in main text from the following two Lemmas: 
(i) Convergence of the first moment of the increment of the macroscopic stochastic process $\mac{M}^{t}$, 
and (ii) Vanishing of the second moment of the increment. 
Intuitively, these ensure that the leading order of the average increment is captured by the ODEs described in Theorem 4.2 and that the stochastic part of the increment of the macroscopic state $\mac{M}^{t}$ vanishes as the input dimension increases.

The whole proof is divided into 4 parts. The first step is to prove the two conditions in the subsequent section. Then, it is demonstrated that these two conditions are sufficient to prove Theorem 4.2. Finally, technical Lemmas that are repeatedly used in the above proofs are summarized. The proof follows the standard scheme of the convergence of stochastic processes \citep{Kushner2009, billingsley2013convergence, wang2018subspace}.

\subsection{Convergence of First Moments of Increment to ODEs}\label{subsec:covergence-first-moments}
We first review the training algorithm of SGD which characterizes a Markov process $X^{t}=(W^{t}, V^{t}, D^{t})$. The specific update rule is given by
\begin{align}
    \B{w}_{m}^{t+1} &= \B{w}_{m}^{t} -\frac{\tau_{W}}{N} \Bigg(\sum_{n=1}^{M} \B{w}_{n} \left(\sqrt{\rho} \sum_{s=1}^{M^{\ast}} c_{s}^{t}d_{ms}^{t} + \sqrt{\eta} \zeta^{t}_{m} \right) \left(\sqrt{\rho} \sum_{s=1}^{M^{\ast}} c_{s}^{t}d_{ns}^{t} + \sqrt{\eta} \zeta^{t}_{n} \right) + (D_{m}^{t} + \lambda) \B{w}_{m}^{t} \notag \\
    &- \left(\sum_{s=1}^{M^{\ast}} \sqrt{\rho} c_{s}^{t} \B{w}^{\ast}_{s} + \sqrt{\eta N} \B{n}^{t} \right)\left(\sum_{s=1}^{M^{\ast}} \sqrt{\rho} c_{s}^{t} d^{t}_{ms} + \sqrt{\eta} \zeta_{m}^{t}  \right)\Bigg), \label{eq:gradient-w} \\
    \B{v}_{m}^{t+1} &= \B{v}_{m}^{t} -\frac{\tau_{V}}{N} \Bigg(\left(\sum_{s=1}^{M^{\ast}} \sqrt{\rho} c_{s}^{t} \B{w}_{s}^{\ast} + \sqrt{\eta N} \B{n}^{t} \right) \Bigg\{\sum_{n} Q_{mn} \left(\sum_{s=1}^{M^{\ast}} \sqrt{\rho} c_{s}^{t} d_{ns}^{t} + \sqrt{\eta} \zeta_{n}^{t}  \right) + \beta \left(\sum_{s=1}^{M^{\ast}} \sqrt{\rho} c_{s}^{t} d_{ms}^{t} + \sqrt{\eta} \zeta_{m}^{t}  \right) \notag \\
    &- \left(\sum_{s=1}^{M^{\ast}} \sqrt{\rho} c_{s}^{t} m_{ms}+ \sqrt{\eta} u_{m}^{t}  \right) \Bigg\}+ \lambda \B{v}_{m}^{t} \Bigg), \label{eq:gradient-v} \\
    D_{m}^{t+1} &= D_{m}^{t} - \frac{\tau_{D}}{2N} \left((Q_{mm}^{t} + \beta) - \frac{\beta}{D_{m}^{t}}  \right), 
\end{align} 
where $\B{w}_{m}^{t}$, $\B{w}_{m}^{\ast}$ and $\B{v}_{m}^{t}$ represent $m$-th columns $W^{t}$, $W^{\ast}$ and $V^{t}$, respectively. 

The following lemma holds for the macroscopic state $\mac{M}^{t}$ characterized by the above updates.

\begin{Lem}
\label{lem:first-incre-macro}
    Under the same assumptions as in Theorem 4.2, for all $t < NT$ the following inequality holds:
    \begin{equation}
    \label{eq:first-moment}
        \mab{E}\left\|\mab{E}_{t} \mac{M}^{t+1} - \mac{M}^{t} - \frac{1}{N} F(\mac{M}^{t})\right\|_{F} \le \frac{C}{N^{3/2}}.
    \end{equation}
\end{Lem}
\begin{proof}
    Recall that $\mac{M}^{t}=(m^{t}, d^{t}, Q^{t}, E^{t}, R^{t}, V^{t}, D^{t}) \in \mathbb{R}^{M \times (2M^{\ast}+5M)}$ is composed of seven matrices. Note that defining $|A|_{F} = \sum_{i=1}^{N}, \sum_{j=1}^{M}|a_{ij}|$ for matrix $A \in \mab{R}^{N \times M}$, the inequality $\|\mac{M}^{t} \|_{F} \le |\mac{M}^{t}|_{F}$ holds,  Thus, the following inequality is sufficient to prove Eq.~\eqref{eq:first-moment}: 
    \begin{equation}
     \label{eq:first-moment-l1norm}
        \mab{E}\left|\mab{E}_{t} \mac{M}^{t+1}_{ij} - \mac{M}_{ij}^{t} - \frac{1}{N} F_{l}(\mac{M}_{ij})\right | \le \frac{C}{N^{3/2}},
    \end{equation}
    where $\mac{M}_{ij}^{t}$ is $ij$ element of  $\mac{M}^{t}$. 
    Subsequently, we show that the above inequality holds for each element of $\mac{M}^{t}$.
    
    For $m^{t}$, the following stronger result is obtained:
    \begin{equation}
    \label{eq:ave-increment-m}
        \mab{E}_{t} m_{ml}^{t+1} - m_{ml}^{t}-\frac{1}{N} F_{m_{ml}}(\mac{M}^{t})=0,
    \end{equation}
    where $F_{m_{ml}}(\mac{M})$ is defined in Eq.~\ref{eq:supp-f-form-m}. 
    This is directly proved by multiplying $(\B{w}^{\ast}_{l})^{\top}/N$ from the left on both sides of Eq.~\ref{eq:gradient-w}, which yields
    \begin{multline}
    \label{eq:nonave-increment-m}
        m_{ml}^{t+1} = m_{ml}^{t} -\frac{\tau_{W}}{N} \Bigg(\sum_{n=1}^{M} m_{nl}^{t} \left(\sqrt{\rho} \sum_{s=1}^{M^{\ast}} c_{s}^{t}d_{ms}^{t} + \sqrt{\eta} \zeta^{t}_{m} \right) \left(\sqrt{\rho} \sum_{s=1}^{M^{\ast}} c_{s}^{t}d_{ns}^{t} + \sqrt{\eta} \zeta^{t}_{n} \right) + (D_{m}^{t} + \lambda) m_{ml}^{t}  \\
        - \left(\sum_{s=1}^{M^{\ast}} \sqrt{\rho} c_{s}^{t} m^{\ast}_{nl} + \sqrt{\eta} u_{l}^{\ast} \right)\left(\sum_{s=1}^{M^{\ast}} \sqrt{\rho} c_{s}^{t} d^{t}_{ms} + \sqrt{\eta} \zeta_{m}^{t}  \right)\Bigg),
    \end{multline}
    where $u_{l}^{\ast} = (\B{w}_{l}^{\ast})^{\top} \B{n}^{t}/\sqrt{N}$.  
    Note that $u_{l}^{\ast}$, $u_{m}^{t}$ and $\zeta_{m}^{t}$ are Gaussian random variables. Then, taking the conditional expectation $\mab{E}_{t}$ on both sides of Eq.~\eqref{eq:nonave-increment-m}, we reach Eq.~\ref{eq:ave-increment-m}.

    Next, we can also get a stronger result for $d^{t}$ given by
    \begin{equation}
    \label{eq:ave-increment-d}
        \mab{E}_{t} d_{ml}^{t+1} - d_{ml}^{t} - \frac{1}{N} F_{d_{ml}}(\mac{M}^{t})=0,
    \end{equation}
    where $F_{d_{ml}}$ is defined in Eq.~\eqref{eq:supp-f-form-d}. 
    This is also proved by multiplying $(\B{w}_{l}^{\ast})^{\top}/N$ from the left on both side of Eq.~\eqref{eq:gradient-v}, which yields
    \begin{multline}
    \label{eq:increment-d}
        d_{ml}^{t+1} = d_{ml}^{t} -\frac{\tau_{V}}{N} \Bigg(\left(\sum_{s=1}^{M^{\ast}} \sqrt{\rho} c_{s}^{t} m_{sl}^{\ast} + \sqrt{\eta} u_{l}^{\ast} \right) \Bigg\{\sum_{n} Q_{mn} \left(\sum_{s=1}^{M^{\ast}} \sqrt{\rho} c_{s}^{t} d_{ns}^{t} + \sqrt{\eta} \zeta_{n}^{t}  \right) + \beta \left(\sum_{s} \sqrt{\rho} c_{s}^{t} d_{ms}^{t} + \sqrt{\eta} \zeta_{m}^{t}  \right)  \\
        - \left(\sum_{s=1}^{M^{\ast}} \sqrt{\rho} c_{s}^{t} m_{ms}+ \sqrt{\eta} u_{m}^{t}  \right) \Bigg\}+ \lambda d_{ml}^{t} \Bigg).
    \end{multline}
    One can also take the conditional expectation $\mab{E}_{t}$ on both sides of Eq.~\eqref{eq:increment-d} since $u_{l}^{t}$, $u_{m}^{t}$, and $\zeta_{m}^{t}$ are Gaussian random variables, leading to Eq.~\eqref{eq:ave-increment-d}.

    Next, for $Q^{t}$, the following inequality holds:
    \begin{equation}
    \label{eq:ave-increment-Q}
        \mab{E}_{t} Q_{mn} - Q_{mn}^{t} - \frac{1}{N} f_{Q_{mn}}(\mac{M}^{t}) \le \frac{C(T)}{N^{\frac{3}{2}}},
    \end{equation}
    where $F_{Q_{mn}}$ is defined in Eq.~\eqref{eq:supp-f-form-Q}. 
    This is proved by evaluating $Q_{mn}^{t+1}=(\B{w}_{m}^{t+1})^{\top} \B{w}_{n}^{t+1}/N$ as follows:
    \begin{align*}
        Q_{mn}^{t+1} &= Q_{mn}^{t} - \frac{\tau_{V}}{N} \left((\nabla_{\B{w}_{m}^{t}} r(X^{t}))^{\top} \B{w}_{n}^{t} + (\B{w}_{m}^{t})^{\top} \nabla_{\B{w}_{m}^{t}} r(X^{t})  \right) + \frac{\tau_{V}^{2}}{N} (\nabla_{\B{w}^{t}_{m}}r(X^{t})^{\top} (\nabla_{\B{w}^{t}_{n}} r(X^{t})) \\
        &=Q_{mn}^{t} - \frac{\tau_{W}}{N} \Bigg\{\sum_{n^{\prime}=1}^{M} Q_{mn^{\prime}}^{t} \left(\sqrt{\rho} \sum_{s=1}^{M^{\ast}} c_{s}^{t}d_{ms}^{t} + \sqrt{\eta} \zeta^{t}_{m} \right) \left(\sqrt{\rho} \sum_{s=1}^{M^{\ast}} c_{s}^{t}d_{n^{\prime}s}^{t} + \sqrt{\eta} \zeta^{t}_{n^{\prime}} \right)  \notag \\
        &+ \sum_{n^{\prime}=1}^{M} Q_{nn^{\prime}}^{t} \left(\sqrt{\rho} \sum_{s=1}^{M^{\ast}} c_{s}^{t}d_{ns}^{t} + \sqrt{\eta} \zeta^{t}_{n} \right) \left(\sqrt{\rho} \sum_{s=1}^{M^{\ast}} c_{s}^{t}d_{n^{\prime}s}^{t} + \sqrt{\eta} \zeta^{t}_{n^{\prime}} \right) \notag \\
        &- \left(\sum_{s=1}^{M^{\ast}} \sqrt{\rho} c_{s}^{t}  m_{ns}^{t} + \sqrt{\eta} u_{n}^{t} \right)\left(\sum_{s=1}^{M^{\ast}} \sqrt{\rho} c_{s}^{t} d^{t}_{ms} + \sqrt{\eta} \zeta_{m}^{t}  \right) - \left(\sum_{s=1}^{M^{\ast}} \sqrt{\rho} c_{s}^{t}  m_{ms}^{t} + \sqrt{\eta} u_{m}^{t} \right)\left(\sum_{s=1}^{M^{\ast}} \sqrt{\rho} c_{s}^{t} d^{t}_{ns} + \sqrt{\eta} \zeta_{n}^{t}  \right) \\
        &+ (D_{n}^{t} + D_{m}^{t} + 2\lambda) Q_{mn}^{t}  \Bigg\} + \frac{\tau_{V}^{2}}{N} \frac{\|\B{n}^{t}\|^{2}}{N} \left(\sum_{s=1}^{M^{\ast}} \sqrt{\rho} c_{s}^{t} d_{ms}^{t} + \sqrt{\eta} \zeta_{m}^{t} \right) \left(\sum_{s=1}^{M^{\ast}} \sqrt{\rho} c_{s}^{t} d_{ns}^{t} + \sqrt{\eta} \zeta_{n}^{t} \right) + \frac{\tau_{V}^{2}}{N^{2}} \Delta(\mac{M}^{t}). 
    \end{align*}
    Also, taking the conditional expectation and using $\mab{E}_{t}|\Delta^{t}(\mac{M}^{t})| \le \sqrt{N}C(T)$, which is proven based on Lemma \ref{lem:macro-bounded}, we can derive Eq.~\ref{eq:ave-increment-Q}.
    Then, the following inequality holds for $E^{t}$:
    \begin{equation}
    \label{eq:ave-increment-E}
        \mab{E}_{t}E_{mn}^{t+1} - E_{mn}^{t} - \frac{1}{N} F_{E_{mn}}(\mac{M}^{t}) \le \frac{C(T)}{N^{\frac{3}{2}}},
    \end{equation}
    where $F_{E_{mn}^{t}}$ is defined in Eq. \eqref{eq:supp-f-form-Q}. This is proved by evaluating $E_{mn}^{t+1}=(\B{v}_{m}^{t+1})^{\top} \B{v}_{n}^{t+1}/N$ as follows:
    \begin{align*}
        E_{mn}^{t+1} &= E_{mn}^{t} - \frac{\tau_{V}}{N} \left((\nabla_{\B{v}_{m}^{t}}r(X^{t}))^{\top} \B{v}_{n}^{t} + (\B{v}_{m}^{t})^{\top} \nabla_{\B{v}_{n}^{t}} r(X^{t})    \right) + \frac{\tau_{V}^{2}}{N} (\nabla_{\B{v}_{m}^{t}} r(X^{t}))^{\top} (\nabla_{\B{v}_{n}^{t}}r(X^{t})), \\
        &= E_{mn}^{t} - \frac{\tau_{V}}{N} \Bigg\{\left(\sum_{s=1}^{M^{\ast}} \sqrt{\rho} c_{s}^{t} d_{ns}^{t} + \sqrt{\eta} \zeta_{n}^{t} \right) \Bigg\{\sum_{n^{\prime}} Q_{mn^{\prime}} \left(\sum_{s=1}^{M^{\ast}} \sqrt{\rho} c_{s}^{t} d_{n^{\prime}s}^{t} + \sqrt{\eta} \zeta_{n^{\prime}}^{t}  \right) + \beta \left(\sum_{s} \sqrt{\rho} c_{s}^{t} d_{ms}^{t} + \sqrt{\eta} \zeta_{m}^{t}  \right) \notag \\
        &- \left(\sum_{s=1}^{M^{\ast}} \sqrt{\rho} c_{s}^{t} m_{ms}^{t}+ \sqrt{\eta} u_{m}^{t}  \right) \Bigg\} + \left(\sum_{s=1}^{M^{\ast}} \sqrt{\rho} c_{s}^{t} d_{ms}^{t} + \sqrt{\eta} \zeta_{m}^{t} \right) \Bigg\{\sum_{n^{\prime}} Q_{nn^{\prime}} \left(\sum_{s=1}^{M^{\ast}} \sqrt{\rho} c_{s}^{t} d_{n^{\prime}s}^{t} + \sqrt{\eta} \zeta_{n^{\prime}}^{t}  \right) \\
        &+ \beta \left(\sum_{s} \sqrt{\rho} c_{s}^{t} d_{ns}^{t} + \sqrt{\eta} \zeta_{n}^{t}  \right) - \left(\sum_{s=1}^{M^{\ast}} \sqrt{\rho} c_{s}^{t} m_{ns}^{t} + \sqrt{\eta} u_{n}^{t}  \right) \Bigg\}+ 2\lambda E_{mn}^{t} \Bigg\} \\
        &+ \frac{\eta \tau_{V}^{2}\|\B{n}^{t}\|^{2}}{N^{2}} \Bigg(\sum_{n^{\prime}} Q_{mn^{\prime}} \left(\sum_{s=1}^{M^{\ast}} \sqrt{\rho} c_{s}^{t} d_{n^{\prime}s}^{t} + \sqrt{\eta} \zeta_{n^{\prime}}^{t}  \right) + \beta \left(\sum_{s} \sqrt{\rho} c_{s}^{t} d_{ms}^{t} + \sqrt{\eta} \zeta_{m}^{t}  \right) \notag \\
        &- \left(\sum_{s=1}^{M^{\ast}} \sqrt{\rho} c_{s}^{t} m_{ms}^{t} + \sqrt{\eta} u_{m}^{t}  \right) \Bigg) \Bigg(\sum_{n^{\prime}} Q_{nn^{\prime}} \left(\sum_{s=1}^{M^{\ast}} \sqrt{\rho} c_{s}^{t} d_{n^{\prime}s}^{t} + \sqrt{\eta} \zeta_{n^{\prime}}^{t}  \right) + \beta \left(\sum_{s} \sqrt{\rho} c_{s}^{t} d_{ns}^{t} + \sqrt{\eta} \zeta_{n}^{t}  \right) \notag \\
        &- \left(\sum_{s=1}^{M^{\ast}} \sqrt{\rho} c_{s}^{t} m_{ns}^{t}+ \sqrt{\eta} u_{n}^{t}  \right) \Bigg) + \frac{\tau_{V}^{2}}{N^{2}} \tilde{\Delta}(\mac{M}^{t}).
    \end{align*}
    Here, one can also take the conditional expectation and use $\mab{E}_{t}|\tilde{\Delta}(\mac{M}^{t})| \le \sqrt{N}C(T)$ that is proven based on Lemma \ref{lem:macro-bounded} and then reach Eq.~\ref{eq:ave-increment-E}.
    
    Next, for $R_{mn}$, the following holds:
    \begin{equation}
    \label{eq:ave-increment-R}
        \mab{E}_{t} R_{mn}^{t} - R_{mn}^{t} - \frac{1}{N} F_{R_{mn}}(\mac{M}^{t}) \le \frac{C(T)}{N^{\frac{3}{2}}},
    \end{equation}
    where $F_{R_{mn}}$ is defined in Eq.~\eqref{eq:supp-f-form-R}. 
    This is proved by evaluating $R_{mn}^{t+1}=(\B{w}_{m}^{t+1})^{\top} \B{v}_{n}^{t+1}/N$ as follows:
    \begin{align*}
        R_{mn}^{t+1} &= R_{mn}^{t}-\frac{\tau_{W}}{N} (\nabla_{\B{w}_{m}^{t}} r(X^{t}))^{\top} \B{v}_{n}^{t} + \frac{\tau_{V}}{N} (\B{w}_{m}^{t})^{\top} \nabla_{\B{v}_{n}^{t}} r(X^{t}) + \frac{\tau_{W} \tau_{V}}{N} (\nabla_{\B{w}_{m}^{t}} r(X^{t}) )^{\top}(\nabla_{\B{v}_{n}^{t}} r(X^{t})), \\
        &= R_{mn}^{t} - \frac{\tau_{W}}{N} \Bigg(\sum_{n^{\prime}=1}^{M} R_{n^{\prime} n} \left(\sqrt{\rho} \sum_{s=1}^{M^{\ast}} c_{s}^{t}d_{ms}^{t} + \sqrt{\eta} \zeta^{t}_{m} \right) \left(\sqrt{\rho} \sum_{s=1}^{M^{\ast}} c_{s}^{t}d_{n^{\prime}s}^{t} + \sqrt{\eta} \zeta^{t}_{n^{\prime}} \right) + (D_{m}^{t} + \lambda)R_{mn}^{t} \notag \\
        &- \left(\sum_{s=1}^{M^{\ast}} \sqrt{\rho} c_{s}^{t} d_{ns}^{t} + \sqrt{\eta} \zeta_{n}^{t} \right)\left(\sum_{s=1}^{M^{\ast}} \sqrt{\rho} c_{s}^{t} d^{t}_{ms} + \sqrt{\eta} \zeta_{m}^{t}  \right)\Bigg) \\
        &- \frac{\tau_{V}}{N} \Bigg(\left(\sum_{s=1}^{M^{\ast}} \sqrt{\rho} c_{s}^{t} m_{ms}^{t} + \sqrt{\eta} u_{m} \right) \Bigg\{\sum_{n^{\prime}} Q_{nn^{\prime}} \left(\sum_{s=1}^{M^{\ast}} \sqrt{\rho} c_{s}^{t} d_{n^{\prime}s}^{t} + \sqrt{\eta} \zeta_{n^{\prime}}^{t}  \right) + \beta \left(\sum_{s} \sqrt{\rho} c_{s}^{t} d_{ns}^{t} + \sqrt{\eta} \zeta_{n}^{t}  \right) \notag \\
        &- \left(\sum_{s=1}^{M^{\ast}} \sqrt{\rho} c_{s}^{t} m_{ns}^{t}+ \sqrt{\eta} u_{n}^{t}  \right) \Bigg\}+ \lambda R_{mn}^{t} \Bigg) \\
        &+ \frac{\tau_{W}\tau_{V}\eta \|\B{n}^{t}\|^{2}}{N^{2}} \left(\sum_{s=1}^{M^{\ast}} \sqrt{\rho} c_{s}^{t} d_{ms}^{t} + \sqrt{\eta} \zeta_{m}^{t}  \right)\Bigg\{\sum_{n^{\prime}} Q_{nn^{\prime}}^{t} \left(\sum_{s=1}^{M^{\ast}} \sqrt{\rho} c_{s}^{t} d_{n^{\prime}s}^{t} + \sqrt{\eta} \zeta_{n^{\prime}}^{t}  \right) + \beta \left(\sum_{s} \sqrt{\rho} c_{s}^{t} d_{ns}^{t} + \sqrt{\eta} \zeta_{n}^{t}  \right) \notag \\
        &- \left(\sum_{s=1}^{M^{\ast}} \sqrt{\rho} c_{s}^{t} m_{ns}^{t}+ \sqrt{\eta} u_{n}^{t}  \right) \Bigg\}  + \frac{\tau_{W}\tau_{V}}{N^{2}} \Omega_{t},
    \end{align*}
    Then, one can take the conditional expectation and use $\mab{E}_{t}|\tilde{\Omega}(\mac{M}^{t})| \le \sqrt{N}C(T)$ that is proven based on Lemma \ref{lem:macro-bounded} and then reach Eq.~\ref{eq:ave-increment-R}.
    
    Lastly, the following stronger result holds for $D_{m}^{t}$ :
    \begin{equation}
    \label{eq:ave-increment-D}
        \mab{E}_{t} D_{m}^{t+1} - D_{m}^{t} - \frac{1}{N} F_{D_{m}}(\mac{M}^{t}) = 0,
    \end{equation}
    where $F_{D_{mn}}$ is defined in Eq.~\eqref{eq:supp-f-form-D}. 
    one can directly obtain as following
    \begin{equation}
        D_{m}^{t+1} = D_{m}^{t} - \tau_{D} \left( (Q_{mm}^{t} + \beta) - \frac{\beta}{D_{m}^{t}} \right). 
    \end{equation}
    Then, one takes the conditional expectation and then reaches Eq.~\ref{eq:ave-increment-D}.
    Combining Eq.~\eqref{eq:ave-increment-m}-\eqref{eq:ave-increment-D}, Eq.~\eqref{eq:first-moment} is proven, which concludes the whole proof.
\end{proof}

\subsection{Convergence of Second Moments of Increment}

We now proceed to bound the second-order moments of the increments.
\begin{Lem}
\label{lem:second-incre-macro}
    Under the same assumption as in Theorem 4.2, for all $t < NT$ the following inequality holds: 
    \begin{equation}
    \label{eq:second-incre-macro}
        \mab{E}\|\mac{M}^{t+1} - \mab{E}_{t}\mac{M}^{t+1}\|_{F}^{2} \le \frac{C(T)}{N^{2}}.
    \end{equation}
\end{Lem}
\begin{proof}
    Note that 
    \begin{align*}
        \mab{E} \|\mac{M}^{t+1}-\mab{E}_{t} \mac{M}^{t+1}\|^{2}_{F} 
        &= \mab{E} \|\mac{M}^{t+1}-\mac{M}^{t}-\mab{E}_{t}(\mac{M}^{t+1}-\mac{M}^{t})\|_{F}^{2}, \\
        &\le \mab{E} \|\mac{M}^{t+1}-\mac{M}^{t}\|_{F}^{2} + \mab{E}\|\mab{E}_{t} \mac{M}^{t+1} - \mac{M}^{t}\|^2_{F}, \\
        &\le \mab{E} \|\mac{M}^{t+1}-\mac{M}^{t}\|_{F}^{2} + \mab{E} \left\|\frac{1}{N} F(\mac{M}^{t}) + \frac{C(T)}{N^{\frac{3}{2}}} \right\|^{2}, \\
        &\le \mab{E} \|\mac{M}^{t+1}-\mac{M}^{t}\|_{F}^{2} + \frac{C(T)}{N^{2}}.
    \end{align*}
    Here the third line is due to Lemma \ref{lem:first-incre-macro}. Thus, it is sufficient to prove that 
    \begin{equation*}
        \mab{E}\|\mac{M}^{t+1}-\mac{M}^{t}\|^{2}_{F} \le \frac{C(T)}{N^{2}}.
    \end{equation*}
    In the following, the second moment of each element in $\mac{M}^{t+1}-\mac{M}^{t}$ will be bounded.

    For $m^{t}$, the following inequality holds: 
    \begin{align}
    \label{eq:second-incre-m}
    \mab{E}(m_{ml}^{t+1}-m_{ml}^{t})^{2} &= \frac{\tau_{W}}{N^{2}} \mab{E}\Bigg[ \Bigg(\sum_{n=1}^{M} m_{nl}^{t} \left(\sqrt{\rho} \sum_{s=1}^{M^{\ast}} c_{s}^{t}d_{ms}^{t} + \sqrt{\eta} \zeta^{t}_{m} \right) \left(\sqrt{\rho} \sum_{s=1}^{M^{\ast}} c_{s}^{t}d_{ns}^{t} + \sqrt{\eta} \zeta^{t}_{n} \right) + (D_{m}^{t} + \lambda) m_{ml}^{t}, \notag \\
    &- \left(\sum_{s=1}^{M^{\ast}} \sqrt{\rho} c_{s}^{t} m^{\ast}_{nl} + \sqrt{\eta} u_{l}^{\ast} \right)\left(\sum_{s=1}^{M^{\ast}} \sqrt{\rho} c_{s}^{t} d^{t}_{ms} + \sqrt{\eta} \zeta_{m}^{t}  \right)\Bigg)^{2} \Bigg] \notag \\
    &\le \frac{C}{N^{2}} \mab{E} \Bigg[\sum_{n, h} m_{nl}^{t}m_{hl}^{t} h(d^{t}_{m}, d^{t}_{n}, E^{t}_{mn}) h(d^{t}_{m}, d^{t}_{h}, E^{t}_{mh}) + (D_{m}^{t}+\lambda)^{2} (m_{ml}^{t})^{2} + h^{2}(m_{l}^{\ast}, d_{m}, d_{ml}) \notag \\
    &+2 (D_{m}^{t}+\lambda)m_{ml}^{t} \left(\sum_{n} m_{nl}^{t} h(d_{m}, d_{n}, E_{mn})- h(m_{l}^{\ast}, d_{m}^{t}, d_{ml}^{t}) \right)  \notag \\
    &-2 h(d^{t}_{m}, d^{t}_{n}, E_{mn}) h(m_{l}^{\ast}, d_{m}^{t}, d_{ml}^{t})    \Bigg] \le \frac{C(T)}{N^{2}}
    \end{align}
    Here, the last line is due to Lemma \ref{lem:macro-bounded}. 

    Next, for $d^{t}$, one can get the following inequality in a similar way:
    \begin{align}
    \label{eq:second-incre-d}
        &\mab{E}(d_{ml}^{t+1}-d_{ml}^{t})^{2} = \frac{\tau_{V}^{2}}{N^{2}} \mab{E} \Bigg[ \Bigg(\left(\sum_{s=1}^{M^{\ast}} \sqrt{\rho} c_{s}^{t} m_{sl}^{\ast} + \sqrt{\eta} u_{l}^{\ast} \right) \Bigg\{\sum_{n} Q_{mn} \left(\sum_{s=1}^{M^{\ast}} \sqrt{\rho} c_{s}^{t} d_{ns}^{t} + \sqrt{\eta} \zeta_{n}^{t}  \right) + \beta \left(\sum_{s} \sqrt{\rho} c_{s}^{t} d_{ms}^{t} + \sqrt{\eta} \zeta_{m}^{t}  \right),  \notag\\
        &- \left(\sum_{s=1}^{M^{\ast}} \sqrt{\rho} c_{s}^{t} m_{ms}+ \sqrt{\eta} u_{m}^{t}  \right) \Bigg\}+ \lambda d_{ml}^{t} \Bigg)^{2} \Bigg] \notag \\
        &\le \frac{C}{N^{2}} \mab{E} \Bigg[h(m_{l}^{\ast}, m_{l}^{\ast}, 1) \Bigg\{\sum_{n, h} Q_{mn} Q_{mh} h(d_{n}^{t}, d_{h}^{t}, E_{nh}^{t}) + \beta^{2} h(d_{m}^{t}, d_{m}^{t}, E_{mm}^{t}) + h(m^{t}_{m}, m^{t}_{m}, Q^{t}_{mm}) \notag \\
        &+ 2\beta \left(\sum_{n} Q_{mn} h(d_{n}, d_{m}, Q_{nm} -\beta h(d_{m}, m_{m}, R_{mm})\right) - \sum_{n} Q_{mn}^{t} h(d_{n}, m_{m}, R_{mn})  \Bigg\} \notag \\
        &+ 2 \lambda d_{ml}^{t} \left(\sum_{n} Q_{mn}^{t} h(m_{l}^{\ast}, d_{n}^{t}, d_{ml}^{t}) + \beta h(m_{l}^{\ast}, d_{m}^{t}, d_{ml}^{t}) - h(m_{l}^{\ast}, m_{m}^{t}, m_{ml}^{t}) \right) + \lambda^{2} (d_{ml}^{t})^{2}    \Bigg] \le \frac{C(T)}{N^{2}}.
    \end{align}
    Here, the last line is also due to Lemma \ref{lem:macro-bounded}.
    Similarly, one can also prove that 
    \begin{align}
        &\mab{E}(Q_{mn}^{t+1}-Q_{mn}^{t})^{2} \le \frac{C(T)}{N^{2}} \label{eq:second-incre-Q},\\
        &\mab{E}(E_{mn}^{t+1}-E_{mn}^{t})^{2} \le \frac{C(T)}{N^{2}} \label{eq:second-incre-E},\\
        &\mab{E}(R_{mn}^{t+1}-R_{mn}^{t})^{2} \le \frac{C(T)}{N^{2}} \label{eq:second-incre-R}, \\
        &\mab{E}(D_{m}^{t+1}-D_{m}^{t})^{2} \le \frac{C(T)}{N^{2}} \label{eq:second-incre-D}. 
    \end{align}
    Combining Eq.~\eqref{eq:second-incre-m}-\eqref{eq:second-incre-D}, Eq.~\eqref{eq:second-incre-macro} is proven, which concludes the whole proof.
\end{proof}

\subsection{Proof of Theorem 4.2}
In this section, we finish the remaining proof of Theorem 4.2 from Lemma \ref{lem:first-incre-macro} and \ref{lem:second-incre-macro} by using the coupling trick.

\begin{proof}
    The proof uses the coupling trick. In particular, we first define a stochastic process $\mac{B}^{t}$ that is coupled with the process $\mac{M}^{t}$ as 
    \begin{equation}
        \mac{B}^{t+1} = \mac{B}^{t} + \frac{1}{N} F(\mac{B}^{t}) + \mac{M}^{t+1} - \mab{E}_{t} \mac{M}^{t+1}
    \end{equation}
    with the deterministic initial condition $\mac{B}^{0} = \bar{\mac{M}}^{0}$. 
    For this stochastic process $\mac{B}^{t}$, the following inequality holds for all $t \le NT$:
    \begin{equation}
        \label{eq:B-bound-coupling}
        \mab{E}\|\mac{B}^{t}- \mac{M}^{t}\|_{F} \le \frac{C(T)}{N^{1/2}}.
    \end{equation}
    This inequality is proved as follows.
    \begin{equation*}
        \mab{E}\|\mac{B}^{t+1}-\mac{M}^{t+1}\|_{F} \le \mab{E}\|\mac{B}^{t} - \mac{M}^{t}\|_{F} + \frac{1}{N} \mab{E}\|F(\mac{B}^{t}) - F(\mac{M}^{t})\|_{F} + \mab{E} \|\mab{E}_{t} \mac{M}^{t+1} -\mac{M}^{t} - \frac{1}{N} F(\mac{M}^{t})\|_{F}.
    \end{equation*}
    From Lemma \ref{lem:first-incre-macro} and Lemma \ref{lem:lipschiz} in subsequent Sec.~\ref{subsec:extra-proofs}, one can get
    \begin{align*}
        \mab{E} \|\mac{B}^{t+1}-\mac{M}^{t+1}\|_{F} 
        &\le \mab{E}\|\mac{B}^{t}-\mac{M}^{t}\|_{F} + L \|\mac{B}^{t}-\mac{M}^{t}\|_{F} + C(T) N^{-\frac{3}{2}} \\
        &\le (1+L N^{-1}) \|\mac{B}^{t}-\mac{M}^{t}\| + C N^{-\frac{3}{2}}. 
    \end{align*}
    Applying this bound iteratively, for all $t \le NT$, one can expand as follows:
    \begin{equation}
        \mab{E}\|\mac{B}^{t} -\mac{M}^{t}\|_{F} \le e^{LT}\left(\mab{E}\|\mac{B}^{0}-\mac{M}^{0}\|_{F} + \frac{C}{L} N^{-\frac{1}{2}}  \right) \le \frac{C(T)}{N^{\frac{1}{2}}}.
    \end{equation}
    For the last inequality, we use the assumption (A.3) in the main text.

    Next, we define a deterministic process $\mac{S}^{t}$ as follows:
    \begin{equation}
        \mac{S}^{t+1} = \mac{S}^{t} + \frac{1}{N} F(\mac{S}^{t})
    \end{equation}
    with the deterministic initial condition $\mac{S}^{0} = \bar{\mac{M}}^{0}$. Similarly, the following inequality holds for all $t \le NT$:
    \begin{equation}
        \label{eq:S-bound-coupling}
        \mab{E}\|\mac{B}^{t} - \mac{S}^{t}\|^{2} \le \frac{C(T)}{N}
    \end{equation}
    To prove this inequality, one can express as 
    \begin{equation*}
        \mab{E}\|\mac{B}^{t+1} - \mac{S}^{t+1}\|_{F} = \mab{E}\|\mac{B}^{t}-\mac{S}^{t}\|^{2} + \frac{1}{N^{2}} \mab{E} \|F(\mac{B}^{t}) - F(\mac{S}^{t}) \|^{2} + \frac{2}{N} \mab{E} (F(\mac{B}^{t}) - F(\mac{S}^{t}))^{\top} (\mac{B}^{t}-\mac{S}^{t}) + \mab{E}\|\mac{M}^{t+1} - \mab{E}_{t} \mac{M}^{t}\|^{2}_{F}.
    \end{equation*}
    Here, one uses the identity given by
    \begin{equation*}
        \mab{E}_{t}(\mac{M}^{t+1}-\mab{E}_{t} \mac{M}^{t})^{\top} (\mac{B}^{t}-\mac{S}^{t}) = \mab{E}_{t}(\mac{M}^{t+1}-\mab{E}_{t} \mac{M}^{t})^{\top} (F(\mac{B}^{t})-F(\mac{S}^{t})) = 0. 
    \end{equation*}
    Then, from Lemma \ref{lem:second-incre-macro} and Lemma \ref{lem:lipschiz} in Sec.~\ref{subsec:extra-proofs} below, one can get following inequality:
    \begin{equation*}
        \mab{E}\|\mac{B}^{t+1}-\mac{S}^{t+1}\|_{F}^{2} \le \left(1+\frac{CL}{N}\right) \mab{E}\|\mac{B}^{t}-\mac{S}^{t}\|_{F}^{2} + \frac{C(T)}{N^{2}}.
    \end{equation*}
    Applying this bound iteratively, for all $t \le NT$, Eq.~\eqref{eq:S-bound-coupling} is proven as follows:
    \begin{equation}
        \mab{E}\|\mac{B}^{t}-\mac{S}^{t}\|_{F}^{2} \le \frac{C(T)}{N}.
    \end{equation}
    Note that $\mac{S}^{t}$ is a standard first-order finite difference approximation of the ODEs with the step size $1/N$. 
    The standard Euler argument implies that 
    \begin{equation}
        \label{eq:M-bounded-coupling}
        \|\mac{S}^{t}-\mac{M}(t)\| \le \frac{C}{N}.
    \end{equation}
    Finally, combining Eq.~\eqref{eq:B-bound-coupling}, \eqref{eq:S-bound-coupling} and \eqref{eq:M-bounded-coupling}, Theorem 4.2 is proven as follows:
    \begin{align*}
        \mab{E}\|\mac{M}^{t} - \mac{M}(t)\| &= \mab{E}\|\mac{M}^{t}-\mac{B}^{t} + \mab{B}^{t}-\mac{S}^{t} + \mac{S}^{t} - \mac{M}(t)\| \\
        &\le \mab{E}\|\mac{M}^{t} -\mac{B}^{t}\| + \mab{E}\|\mac{B}^{t} - \mac{S}^{t}\| + \mab{E}\|\mac{S}^{t}-\mac{M}(t)\| \\
        &\le \mab{E}\|\mac{M}^{t} -\mac{B}^{t}\| + (\mab{E}\|\mac{B}^{t} - \mac{S}^{t}\|^{2})^{\frac{1}{2}} + \mab{E}\|\mac{S}^{t}-\mac{M}(t)\| \\
        &\le \frac{C(T)}{N^{\frac{1}{2}}}.
    \end{align*}
\end{proof}

\subsection{Extra Proofs}\label{subsec:extra-proofs}
In this section, we complete the extra technical lemmas related to the proofs in the previous section.

\subsubsection{Bound for Micoroscopic State}\label{subsubsec:bound-micro}
\begin{Lem}
\label{lem:micro-bound}
    Under the same assumption as in Theorem 4.2, for all $t\le NT$ and $i=1, \ldots N$, the following inequality holds: 
    \begin{equation}
    \label{eq:paramter-bound}
        \mab{E} \left(\sum_{n=1}^{M} (W_{in}^{t})^{4} + \sum_{n=1}^{M} (V_{in}^{t})^{4} + \sum_{n=1}^{M} (D_{n}^{t})^{4} \right) \le C(T).
    \end{equation}
\end{Lem}
\begin{proof}
We first prove $\mab{E} (W_{il}^{t})^{4} \le C(T)$. Note that one can expand as follows:
\begin{multline}
    \label{eq:decompose-w}
    \mab{E} (W^{t+1}_{il})^{4} - \mab{E} (W^{t}_{il})^{4} = 4 \mab{E}[(W_{in}^{t})^{3}\mab{E}_{t}(W_{in}^{t+1}-W_{in}^{t})] + 6 \mab{E}[(W_{in}^{t})^{2}\mab{E}_{t}(W_{in}^{t+1}-W_{in}^{t})^{2}] \\
    +4 \mab{E}[W_{in}^{t}\mab{E}_{t}(W_{in}^{t+1}-W_{in}^{t})^{3}] + \mab{E}[\mab{E}_{t}(W_{in}^{t+1}-W_{in}^{t})^{4}]
\end{multline}
From Eq.~\ref{eq:gradient-w} and the triangle inequality, the following inequality holds for $\gamma = 1, 2, 3$ and $4$:
\begin{equation}
\label{eq:triangle-w}
\mab{E}_{t}(W_{in}^{t+1}-W_{in}^{t})^{\gamma} 
\le \frac{C}{N^{\gamma}}\Bigg[\left|\sum_{n^{\prime}=1}^{M} W_{in^{\prime}}^{t}\left(\rho \sum_{s}d_{ms}^{t} d_{ns}^{t} + \eta E_{mn}^{t}\right) \right|^{\gamma} + |(D_{m}^{t} + \lambda)W_{in}|^{\gamma} + \left|\sum_{s}W_{is}^{\ast} d_{ns}^{t}  \right|^{\gamma} + |V^{t}_{in}|^{\gamma}\Bigg].   
\end{equation}
Substituting Eq.~\eqref{eq:triangle-w} into Eq.~\eqref{eq:decompose-w}, we have
\begin{multline}
\label{eq:recurrence-w}
\mab{E} (W^{t+1}_{in})^{4} - \mab{E} (W^{t}_{in})^{4} \le \frac{C}{N} \mab{E} \Bigg[\left|(W_{in}^{t})^{3}\sum_{n^{\prime}=1}^{M} W_{in^{\prime}}^{t}\left(\rho \sum_{s}d_{ms}^{t} d_{ns}^{t} + \eta E_{mn}^{t}\right) \right| + |(D^{t}_{m} + \lambda)(W_{in}^{t})^{4}| \\
+ \left|(W_{in}^{t})^{3}\sum_{s}W_{is}^{\ast} d_{ns}^{t} \right| + |(W_{in}^{t})^{3}V^{t}_{in}|\Bigg] + \mac{O}(N^{-2}).
\end{multline}

For $V_{il}^{t}$, one can obtain the following:
\begin{multline}
    \label{eq:decompose-v}
    \mab{E} (V^{t+1}_{il})^{4} - \mab{E} (V^{t}_{il})^{4} = 4 \mab{E}[(V_{in}^{t})^{3}\mab{E}_{t}(V_{in}^{t+1}-V_{in}^{t})] + 6 \mab{E}[(V_{in}^{t})^{2}\mab{E}_{t}(V_{in}^{t+1}-V_{in}^{t})^{2}] \\
    +4 \mab{E}[V_{in}^{t}\mab{E}_{t}(V_{in}^{t+1}-V_{in}^{t})^{3}] + \mab{E}[\mab{E}_{t}(V_{in}^{t+1}-V_{in}^{t})^{4}].
\end{multline}
From Eq.~\ref{eq:gradient-v} and the triangle inequality, the following inequality holds for $\gamma=1, 2, 3$ and $4$:
\begin{multline}
\label{eq:triangle-v}
    \mab{E}_{t}(V_{in}^{t+1}-V_{in}^{t})^{\gamma} \le \frac{C}{N^{\gamma}} \Bigg[\left|\sum_{s} W_{is}^{\ast}\left(\sum_{n^{\prime}} Q_{nn^{\prime}}^{t}\sum_{s} d_{n^{\prime}s}^{t} + \beta \sum_{s} d_{ns}^{t} - \sum_{s} m_{ns}^{t}   \right)\right|^{\gamma} \\
    + \left|\sum_{n^{\prime}} Q_{nn^{\prime}}^{t}V_{in^{\prime}}^{t}\right|^{\gamma} + |W_{in}^{t}|^{\gamma} +| V_{in}^{t}|^{\gamma} \Bigg].
\end{multline}
Substituting Eq.~\eqref{eq:triangle-v} into Eq.~\eqref{eq:decompose-v}, one can obtain the following:
\begin{multline}
\label{eq:recurrence-v}
\mab{E} (V^{t+1}_{il})^{4} - \mab{E} (V^{t}_{il})^{4}\le \frac{C}{N} \Bigg[\left|(V_{in}^{t})^{3}\sum_{s} W_{is}^{\ast}\left(\sum_{n^{\prime}} Q_{nn^{\prime}}^{t}\sum_{s} d_{n^{\prime}s}^{t} + \beta \sum_{s} d_{ns}^{t} - \sum_{s} m_{ns}^{t}   \right)\right| \\
+ \left|(V_{in}^{t})^{3}\sum_{n^{\prime}} Q_{nn^{\prime}}V_{in^{\prime}}^{t}\right| + |(V_{in}^{t})^{3}W_{in}^{t}| +| (V_{in}^{t})^{4}|^{\gamma} \Bigg] + \mac{O}(N^{-2}).   
\end{multline}
Similarly, one can also get the following inequality:
\begin{equation}
    \label{eq:recurrence-D}
    \mab{E}(D_{n}^{t+1})^{4} - \mab{E}(D_{n}^{t})^{4} \le \frac{C}{N}\mab{E}\left[\left|(D_{n}^{t})^{3}Q_{nn}^{t}\right|+|(D_{n}^{t})^{2}|  \right].
\end{equation}

Combining Eq.~\eqref{eq:recurrence-w}, Eq.~\eqref{eq:recurrence-v} and Eq.~\eqref{eq:recurrence-D}, the following inequality holds:
\begin{multline}
    \mab{E}\left[(W_{in}^{t+1})^{4} + (V_{in}^{t+1})^{4} +(D_{n}^{t+1})^{4} \right] -     \mab{E}\left[(W_{in}^{t})^{4} + (V_{in}^{t})^{4} + (D_{n}^{t})^{4} \right] \\
    \le \frac{C}{N} \Bigg[\left|(W_{in}^{t})^{3}\sum_{n^{\prime}=1}^{M} W_{in^{\prime}}^{t}\left( \sum_{s}d_{ms}^{t} d_{ns}^{t} +  E_{mn}^{t}\right) \right| + |D_{m}^{t}(W_{in}^{t})^{4}| \\
    + \left|(W_{in}^{t})^{3}\sum_{s}W_{is}^{\ast} d_{ns}^{t}  \right| + |(W_{in}^{t})^{3}V_{in}^{t}| + \left|(V_{in}^{t})^{3}\sum_{s} W_{is}^{\ast}\left(\sum_{n^{\prime}} Q_{nn^{\prime}}^{t}\sum_{s} d_{n^{\prime}s}^{t} + \beta \sum_{s} d_{ns}^{t} - \sum_{s} m_{ns}^{t}   \right)\right| \\
    + \left|(V_{in}^{t})^{3}\sum_{n^{\prime}} Q_{nn^{\prime}}^{t}V_{in^{\prime}}^{t}\right| + |(V_{in}^{t})^{3}W_{in}^{t}|^{\gamma} +| (V_{in}^{t})^{4}| + |(D_{n}^{t})^{3} Q_{nn}^{t}| + |(D_{n}^{t})^{2}|  \Bigg].
\end{multline}
Using the above inequality iteratively, one can get
\begin{multline*}
    \mab{E}\left[(W_{in}^{t})^{4} + (V_{in}^{t})^{4} + (D_{n}^{t})^{4} \right] \le C(T) \Bigg[\left|(W_{in}^{0})^{3}\sum_{n^{\prime}=1}^{M} W_{in^{\prime}}^{0}\left( \sum_{s}d_{ms}^{0} d_{ns}^{0} +  E_{mn}^{0}\right) \right| + |D_{m}^{0}(W_{in}^{0})^{4}| \\
    + \left|(W_{in}^{0})^{3}\sum_{s}W_{is}^{\ast} d_{ns}^{0}  \right| + |(W_{in}^{0})^{3}V_{in}^{0}| + \left|(V_{in}^{0})^{3}\sum_{s} W_{is}^{\ast}\left(\sum_{n^{\prime}} Q_{nn^{\prime}}^{0}\sum_{s} d_{n^{\prime}s}^{0} + \beta \sum_{s} d_{ns}^{0} - \sum_{s} m_{ns}^{0}   \right)\right| \\
    + \left|(V_{in}^{0})^{3}\sum_{n^{\prime}} Q_{nn^{\prime}}^{0}V_{in^{\prime}}^{0}\right| + |(V_{in}^{0})^{3}W_{in}^{0}|^{\gamma} +| (V_{in}^{0})^{4}| + |(D_{n}^{0})^{3} Q_{nn}^{0}| + |(D_{n}^{0})^{2}|  \Bigg].
\end{multline*}
We now reach Eq.~\eqref{eq:paramter-bound} since initial microscopic states are bounded, i.e., $\mab{E}[\sum_{n=1}^{M} \{(W_{in}^{0})^{4} + (V_{in}^{0})^{4} + (D_{n}^{0})^{4}\}+\sum_{n=1}^{M^{\ast}} W_{in}^{\ast}] \le C$, because of the assumption (A.4).
\end{proof}

\subsubsection{Bound for Macroscopic State} \label{subsub:bound-macro}
\begin{Lem}
\label{lem:macro-bounded}
    Under the same assumption as in Theorem 4.2, for all $t\le NT$, the following inequality holds:
    \begin{align}
        \mab{E}\|Q^{t}\|^{2}_{F} \le C(T),~~\mab{E}\|E^{t}\|^{2}_{F} \le C(T),~~\mab{E}\|R^{t}\|^{2}_{F} \le C(T),~~\mab{E}\|m^{t}\|^{2}_{F} \le C(T), ~~ \mab{E}\|d^{t}\|^{2}_{F} \le C(T).
    \end{align}
\end{Lem}

\begin{proof}
    It is a direct consequence of Lemma \ref{lem:micro-bound}.
    For $Q_{nn}^{t}$, using H\"{o}lder's inequality, one can get
    \begin{align*}
        \mab{E}(Q_{nn}^{t})^{2} &= \frac{1}{N^{2}}\mab{E} \left(\sum_{i=1}^{n} W_{in}^{t}W_{in}^{t}\right)^{2} \\
        &\le \frac{1}{N} \mab{E} \sum_{i=1}^{N} (W_{in}^{t})^{4} \le C(T).
    \end{align*}
    The last line is based on Lemma \ref{lem:micro-bound}.
    
    For $Q_{mn}^{t},~m\neq n$, using Cauchy-Schwartz inequality and H\"{o}lder's inequality, one can get
    \begin{align*}
        \mab{E} Q_{mn}^{t} &= \frac{1}{N^{2}} \mab{E} \left(\sum_{i=1}^{N} W_{im}^{t} W_{in}^{t}  \right)^{2} \\
        &\le \frac{1}{N^{2}}\mab{E} \left(\sum_{i=1}^{N} (W_{im}^{t})^{2}  \right) \left(\sum_{i=1}^{N} (W_{in}^{t})^{2}  \right) \\
        &\le \frac{1}{N^{2}}\sqrt{\mab{E} \left(\sum_{i=1}^{N} (W_{im}^{t})^{2}  \right)^{2} \left(\sum_{i=1}^{N} (W_{in}^{t})^{2}  \right)^{2}} \\
        &\le \frac{1}{N}\sqrt{\mab{E} \sum_{i=1}^{N} (W_{im}^{t})^{4} \sum_{i=1}^{N} (W^{t}_{in})^{4}} \le C(T),
    \end{align*}
    where in reaching the last line, we use Lemma \ref{lem:micro-bound}. Then, we get $\mab{E}\|Q^{t}\|_{F} \le C(T)$. The rest bound of $\mab{E}\|E^{t}\|^{2}_{F}$, $\mab{E}\|R^{t}\|^{2}_{F}$, $\mab{E}\|m^{t}\|^{2}_{F}$ and $\mab{E}\|d^{t}\|^{2}_{F}$ can also be directly verified using the Cauchy-Schwartz inequality and H\"{o}lder's inequality and Lemma \ref{lem:micro-bound}.
\end{proof}

\subsubsection{Lipschitzness of ODEs}
\label{subsubsec:lipschitz-F}
\begin{Lem}
\label{lem:positive-negative-D}
Under the same assumption as in Theorem 4.2, for all $t \le NT$, $D^{t} \neq 0_{N \times N}$ holds.
\end{Lem}
\begin{proof}
Consider the ODE in Eq.~\ref{eq:supp-f-form-D}:
\begin{equation*}
    \frac{dD_{m}(t)}{dt} = \tau_{D} \left(\frac{\beta}{D_{m}(t)} - (Q_{mm}(t) + \beta) \right),
\end{equation*}
where $\tau_{D}, \beta \ge 0$ and $\forall t \le NT, Q_{mm}(t) \ge 0$ by definition. We show the behavior of the solution $D_{m}(t)$ based on its initial condition. For $D_{m}(0) > 0$, the term $\tau_{D}(\beta/D_{m}(t) - (Q_{mm}(t)+\beta))$ is positive as $D_{m}(t)$ approaches zero and negative as $D_{m}(t)$ grows to positive infinity. Consequently, if $D_{m}(t)$ attempts to approach zero, $dD_{m}(t)/dt$ becomes positive, indicating that $D_{m}(t)$ increase, and thus does not cross zero. Similarly, if $D_{m}(t)$ becomes very large, $dD_{m}(t)/dt$ becomes negative, causing $D_{m}(t)$ to decrease but remain positive. Therefore, given the initial condition $D_{m}(0) > 0$, $D_{m}(t)$ remains positive for all $t > 0$. Similarly, we can show that, given the initial condition $D_{m}(0) < 0$, $D_{m}(t)$ remains negative for all $t > 0$.
\end{proof}

\begin{Lem}
    \label{lem:lipschiz}
    Under the same assumption as Theorem 4.2, $F(\mac{M})$ is a Lipschitz function.
\end{Lem}
\begin{proof}
    It suffices to verify each component of gradient $\nabla F(\mac{M})$ is bounded. 
    Eq.~\eqref{eq:supp-f-form-m}-\eqref{eq:supp-f-form-R} are linear functions with respect to $\mac{M}$ and then following inequality holds for $\forall \mac{M}$:
    \begin{align*}
        &\|\nabla_{\mac{M}} F_{m_{ml}}(\mac{M})\| \le L_{m_{ml}}(\mac{M}),~~\|\nabla_{\mac{M}} F_{d_{ml}}(\mac{M})\| \le L_{d_{ml}}(\mac{M}),~~\|\nabla_{\mac{M}} F_{Q_{mn}}(\mac{M})\| \le L_{Q_{mn}}(\mac{M}), \\
        &\|\nabla_{\mac{M}} F_{E_{mn}}(\mac{M})\| \le L_{E_{mn}}(\mac{M}), ~~\|\nabla_{\mac{M}} F_{R_{mn}}(\mac{M})\| \le L_{R_{mn}}(\mac{M}),
    \end{align*}
    where $L_{m_{ml}}(\mac{M})$, $ L_{d_{ml}}(\mac{M})$, $L_{Q_{mn}}(\mac{M})$, $L_{E_{mn}}(\mac{M})$ and $L_{R_{mn}}(\mac{M})$ are constants depending on $\mac{M}$. We can show the constants are bounded based on Lemma \ref{lem:micro-bound}. Thus, the functions satisfy the Lipschitz condition. For $F_{D_{m}}(\mac{M})$, gradient norm is given by
    \begin{equation}
        \|\nabla_{\mac{M}} F_{D_{m}}(\mac{M})\| = \tau \sqrt{1+ \frac{\beta^{2}}{D_{m}^{4}}}.
    \end{equation}
    The left-hand side is also bounded 
    since Lemma \ref{lem:positive-negative-D} indicates that for all $m=1, \ldots, M$, $D_{m}(t) \neq 0$ for any $t >0$. Thus, $F_{D_{m}}(\mac{M})$ also satisfy the Lipschitz condition. 
\end{proof}

\section{Local Stability Analysis of Fixed Points of ODEs}\label{sec:local-stability-fixed-point}
In this section, we provide additional details on the local stability analysis of the ODEs. In what follows, we will omit straightforward calculations related to the eigenvalue computations.

\subsection{Stability Analysis of Model-Matched Case}
For the model-matched case, the macroscopic state is described by 6 variables. For the sake of simplicity, we only consider the case $\lambda=0$ and small learning limit $\tau =\tau_{W}=\tau_{V}=\tau_{D}$. 
The fixed points are given by the condition $d\mac{M}/dt=0$. From Eq.~\eqref{eq:supp-f-form-m}-\eqref{eq:supp-f-form-D}, the fixed point equations given by
\begin{align}
\label{eq:model-matched-fixpoint-eq}
\begin{cases}
    F_{m_{11}}(\mac{M}) = \tau \left(d_{11}(\rho+\eta)-m_{11}(\rho d_{11}^{2} + \eta E_{11} + D_{11}) \right)=0 \\
    F_{d_{11}}(\mac{M}) =\tau (\rho+\eta) (m_{11} - (Q_{11}+\beta) d_{11}) = 0 \\
    F_{Q_{11}}(\mac{M})= 2 \tau \left((\rho m_{11} d_{11} + \eta R_{11}) - Q_{11} (\rho d_{11}^{2} + \eta E_{11} + D_{1})  \right) = 0 \\
    F_{E_{11}}(\mac{M}) = 2 \tau \left((\rho m_{11}d_{11} + \eta R_{11}) - (Q_{11}+\beta)(\rho d_{11}^{2} + \eta E_{11}) \right) =0 \\
    F_{R_{11}}(\mac{M}) = \tau \left((1-R_{11})(\rho d_{11}^{2} + \eta E_{11}) - D_{1} R_{11} + (\rho m_{11}^{2} + \eta Q_{11}) - (Q_{11} + \beta)(\rho m_{11} d_{11} + \eta R_{11})  \right)=0 \\
    F_{D_{1}}(\mac{M})= \tau \left(\frac{\beta}{D_{1}} - (Q_{11}+\beta)  \right) = 0,
\end{cases}
\end{align}
where $\mac{M}=(m_{11}, d_{11}, Q_{11}, E_{11}, R_{11}, D_{1})$ are the stationary macroscopic state. The local stability of a fixed point is identified by whether the Jacobian matrix
\begin{equation}
    J(\mac{M}) \delequal 
    \begin{pmatrix}
    \frac{\partial F_{m_{11}}}{\partial m_{11}} & \frac{\partial F_{m_{11}}}{\partial d_{11}} &
    \frac{\partial F_{m_{11}}}{\partial Q_{11}} &
    \frac{\partial F_{m_{11}}}{\partial E_{11}} &
    \frac{\partial F_{m_{11}}}{\partial R_{11}} &
    \frac{\partial F_{m_{11}}}{\partial D_{1}} \\
    \frac{\partial F_{d_{11}}}{\partial m_{11}} & 
    \frac{\partial F_{d_{11}}}{\partial d_{11}} &
    \frac{\partial F_{d_{11}}}{\partial Q_{11}} &
    \frac{\partial F_{d_{11}}}{\partial E_{11}} &
    \frac{\partial F_{d_{11}}}{\partial R_{11}} &
    \frac{\partial F_{d_{11}}}{\partial D_{1}} \\
    \frac{\partial F_{Q_{11}}}{\partial m_{11}} & 
    \frac{\partial F_{Q_{11}}}{\partial d_{11}} &
    \frac{\partial F_{Q_{11}}}{\partial Q_{11}} &
    \frac{\partial F_{Q_{11}}}{\partial E_{11}} &
    \frac{\partial F_{Q_{11}}}{\partial R_{11}} &
    \frac{\partial F_{Q_{11}}}{\partial D_{1}} \\
    \frac{\partial F_{E_{11}}}{\partial m_{11}} & 
    \frac{\partial F_{E_{11}}}{\partial d_{11}} &
    \frac{\partial F_{E_{11}}}{\partial Q_{11}} &
    \frac{\partial F_{E_{11}}}{\partial E_{11}} &
    \frac{\partial F_{E_{11}}}{\partial R_{11}} &
    \frac{\partial F_{E_{11}}}{\partial D_{1}} \\
    \frac{\partial F_{R_{11}}}{\partial m_{11}} & 
    \frac{\partial F_{R_{11}}}{\partial d_{11}} &
    \frac{\partial F_{R_{11}}}{\partial Q_{11}} &
    \frac{\partial F_{R_{11}}}{\partial E_{11}} &
    \frac{\partial F_{R_{11}}}{\partial R_{11}} &
    \frac{\partial F_{R_{11}}}{\partial D_{1}} \\
    \frac{\partial F_{D_{1}}}{\partial m_{11}} & 
    \frac{\partial F_{D_{1}}}{\partial d_{11}} &
    \frac{\partial F_{D_{1}}}{\partial Q_{11}} &
    \frac{\partial F_{D_{1}}}{\partial E_{11}} &
    \frac{\partial F_{D_{1}}}{\partial R_{11}} &
    \frac{\partial F_{D_{1}}}{\partial D_{1}} \\
    \end{pmatrix}
\end{equation}
has eigenvalue with non-negative real part or not. 
Solving Eq.~\ref{eq:model-matched-fixpoint-eq} and computing the eigenvalues of the Jacobian, one easily finds that fixed points other than two cases have positive eigenvalues for any $\beta$, $\rho$ and $\eta$, indicating that they are unstable fixed points. 
Subsequently, we focus on the two cases. 
In the following, the shorthand expression $P=\eta+\rho$ is employed.

\paragraph{Type (1): Posterior Collapsed Fixed Point}
It is easy to verify that 
\begin{equation}
    m^{\ast}_{11}=d^{\ast}_{11}=Q^{\ast}_{11}=E^{\ast}_{11}=R^{\ast}_{11}=0,~~ D^{\ast}_{1}=1
\end{equation}
is a solution of Eq.~\eqref{eq:model-matched-fixpoint-eq}. This fixed point indicates that the VAE encounters a posterior collapse. From a straightforward eigenvalue computation, the six eigenvalues can be expressed as follows:
\begin{align*}
    \lambda_{1}/\tau &= - \frac{\beta}{2},~~~\lambda_{2}/\tau = -(1+\beta \eta) \\
    \lambda_{3}/\tau &= - \left(1+ \beta \eta + \sqrt{(1+\beta \eta)^{2} + 4\eta (\eta-\beta)}  \right),~~~\lambda_{4}/\tau = - \left(1+ \beta \eta - \sqrt{(1+\beta \eta)^{2} + 4\eta (\eta-\beta)}  \right) \\
    \lambda_{5}/\tau &= -\frac{1}{2} \left(1+\beta P + \sqrt{(1+\beta P)^{2}+4 P(P -\beta)}  \right),~~~
    \lambda_{6}/\tau = -\frac{1}{2} \left(1+\beta P - \sqrt{(1+\beta P)^{2}+4 P(P -\beta)} \right)
\end{align*}
Here, $\lambda_{4}$ is positive when $\beta < \eta$, $\lambda_{6}$ is when $\beta < P$ is positive and the others are negative for any $\beta$, $\rho$ and $\rho$. Thus, type (1) fixed point is stable if $P<\beta$. Moreover, all other fixed points are unstable when $P<\beta$, which indicates that a threshold of the posterior collapse is $\beta=P$.

\paragraph{Type (2): Learnable Fixed Point}
The fixed points equation Eq.~\eqref{eq:model-matched-fixpoint-eq} have following solution:
\begin{align}
    &m_{11}^{\ast} = \pm \sqrt{P-\beta},~d_{11}^{\ast} = \pm \frac{\sqrt{P-\beta}}{P},~Q_{11}^{\ast}=P-\beta,~E_{11}^{\ast}= \frac{P-\beta}{P^{2}},~R_{11}^{\ast} = \frac{P-\beta}{P},~D_{1}^{\ast}= \frac{\beta}{P}
\end{align}
The Jacobian of this fixed point possesses six eigenvalues. The three eigenvalues of them can be expressed as follows:
\begin{align*}
    \lambda_{1}/\tau  &= - (1 + \eta P) \\
    \lambda_{2}/\tau  &= -\left(1+\eta P+ \sqrt{(1+\eta P)^{2}-4\eta \rho}  \right) \\
    \lambda_{3}/\tau  &= -\left(1+\eta P - \sqrt{(1+\eta P)^{2}-4\eta \rho}  \right) 
\end{align*}
This three eigenvalues are negative for any $\beta$, $\eta$ and $\rho$. 
The other three eigenvalues can be expressed as the solutions to the following equation:
\begin{equation*}
    \left(\frac{\lambda}{\tau}\right)^{3} + P^{2}\left(P^{2}+2(1+P^{2})\beta \right) \left(\frac{\lambda}{\tau}\right)^{2} + 2 P^{4} \beta\left(P^{2}(1+P^{2})-8\beta^{3}+2(1+4P)\beta^{2} - 2P \beta   \right) \frac{\lambda}{\tau} + 8P^{8}(P-\beta)\beta^{3}=0.
\end{equation*} 
One of the solutions to this equation is positive when $P > \beta$. Furthermore, by substituting $\beta = P$, the equation can be expressed as
\begin{equation*}
\frac{\lambda}{\tau} \left(\frac{\lambda}{\tau}+P^{4}\right)\left(\frac{\lambda}{\tau} + 2P^{3}(1+P^{2}) \right)=0,    
\end{equation*}
indicating that $\lambda = 0$ when $\beta = P$. Thus, type (2) fixed point is stable when $\beta \le \rho+\eta$.

\subsection{Stability Analysis of Model-Mismatched Case}
For the model-mismatched case, the macroscopic state is described by 16 variables. 
For the sake of simplicity, we also consider the case $\lambda=0$ and small learning limit $\tau =\tau_{W}=\tau_{V}=\tau_{D}$. 
The specific fixed-point equations and their Jacobians can be derived from Eq.~\eqref{eq:supp-f-form-m}-\eqref{eq:supp-f-form-D}, just as in the model-matched case. However, they are not displayed here due to their length.
Similarly, all fixed points other than three cases are unstable fixed points as in the model-matched case, as the eigenvalues of their Jacobians take positive values for any $\beta$, $\rho$ and $\eta$.
Subsequently, we focus on the three types in detail.

\paragraph{Type (1): Posterior Collapsed Fixed Point}
It is easy to verify that the following state is a solution of the ODEs:
\begin{equation*}
m^{\ast}=d^{\ast}=\B{0}_{2},~~Q^{\ast}=E^{\ast}=R^{\ast}=\B{0}_{2 \times 2},~~D^{\ast}=\B{1}_{2}
\end{equation*}
The eigenvalues of the Jacobian can be expressed as follows:
\begin{align*}
    &\frac{\lambda_{1}}{\tau}=\frac{\lambda_{2}}{\tau} = - \frac{\beta}{2},~~\frac{\lambda_{3}}{\tau}=\frac{\lambda_{4}}{\tau} = \frac{\lambda_{5}}{\tau}=\frac{\lambda_{6}}{\tau} = - (1 + \beta \eta) \\
    &\frac{\lambda_{7}}{\tau}=\frac{\lambda_{8}}{\tau} = \frac{\lambda_{9}}{\tau} = - \left(1 + \beta \eta + \sqrt{(1+\beta \eta)^{2} + 4\eta(\eta-\beta)}  \right) \\
    &\frac{\lambda_{10}}{\tau}=\frac{\lambda_{11}}{\tau} = \frac{\lambda_{12}}{\tau} = - \left(1 + \beta \eta - \sqrt{(1+\beta \eta)^{2} + 4\eta(\eta-\beta)}  \right) \\
    &\frac{\lambda_{13}}{\tau}=\frac{\lambda_{14}}{\tau} = - \frac{1}{2} \left(1+\beta P + \sqrt{(1+ \beta P)^{2} + 4 P(P-\beta)}  \right) \\
    &\frac{\lambda_{15}}{\tau}=\frac{\lambda_{16}}{\tau} = - \frac{1}{2} \left(1+\beta P - \sqrt{(1+ \beta P)^{2} + 4 P(P-\beta)}  \right)
\end{align*}
These eigenvalue are positive when $\rho + \eta < \beta$, zero when $\rho + \eta=\beta$ and negative when $\rho+\eta > \beta$ as in the model-matched case. 
Thus this fixed solution is stable when $\rho + \eta \le \beta$.

\paragraph{Type (2): Overfitting Fixed Point}
The fixed point equations have the following solution:
\begin{align*}
    &m^{\ast} = \left(\pm \sqrt{P- \beta}, 0 \right),~ d^{\ast} = \left(\pm \frac{\sqrt{P - \beta}}{P}, 0 \right) \\
    &Q^{\ast} = 
    \begin{pmatrix}
    P-\beta & 0 \\
    0                      & \eta - \beta
    \end{pmatrix},~
    E^{\ast} =
    \begin{pmatrix}
        \frac{P-\beta}{P^{2}} & 0 \\
        0 &\frac{\eta-\beta}{\eta^{2}} 
    \end{pmatrix},~
    R^{\ast} = 
    \begin{pmatrix}
        \frac{P-\beta}{P} & 0 \\
        0 & \frac{\eta-\beta}{\eta}
    \end{pmatrix},~
    D^{\ast} = \left(\frac{\beta}{P}, \frac{\beta}{\eta}  \right)
\end{align*}
and 
\begin{align*}
    &m^{\ast} = \left(0, \pm \sqrt{P- \beta} \right),~ d^{\ast} = \left(0, \pm \frac{\sqrt{P - \beta}}{\rho + \eta} \right) \\
    &Q^{\ast} = 
    \begin{pmatrix}
    \eta-\beta & 0 \\
    0                      & P - \beta
    \end{pmatrix},~
    E^{\ast} =
    \begin{pmatrix}
        \frac{\eta-\beta}{\eta^{2}} & 0 \\
        0 &\frac{P-\beta}{P^{2}} 
    \end{pmatrix},~
    R^{\ast} = 
    \begin{pmatrix}
        \frac{\eta -\beta}{\eta} & 0 \\
        0 & \frac{P-\beta}{P}
    \end{pmatrix},~
    D^{\ast} = \left(\frac{\beta}{\eta}, \frac{\beta}{P}  \right)
\end{align*}
The eigenvalues of the Jacobian can be expressed as follows:
\begin{align*}
    &\frac{\lambda_{1}}{\tau} = - 2(1+\eta^{2}),~~~\frac{\lambda_{2}}{\tau} = -(1+\eta P) \\
    &\frac{\lambda_{3}}{\tau} = -\frac{1}{2} \left(1+\eta P + 2(1+\eta) + \sqrt{(1+\eta P)^{2}-4\eta \rho}  \right),~~\frac{\lambda_{4}}{\tau} = -\frac{1}{2} \left(1+\eta P + 2(1+\eta) - \sqrt{(1+\eta P)^{2}-4\eta \rho}  \right) \\
    &\frac{\lambda_{5}}{\tau} = - \left(1+\eta P + \sqrt{(1+\eta P)^{2}-4\eta \rho}  \right),~~~\frac{\lambda_{6}}{\tau} = - \left(1+\eta P - \sqrt{(1+\eta P)^{2}-4\eta \rho}  \right) \\
    &\frac{\lambda_{7}}{\tau} = \frac{\lambda_{8}}{\tau} = -\frac{1}{2} \left(1+\eta P+ \sqrt{(1+\eta P)^{2} + 8\beta\left(\sqrt{(\beta-\eta)(\beta-P)} + \beta - P + \frac{\rho}{2} \right)}  \right) \\
    &\frac{\lambda_{9}}{\tau} = \frac{\lambda_{10}}{\tau} = -\frac{1}{2} \left(1+\eta P - \sqrt{(1+\eta P)^{2} + 8\beta\left(\sqrt{(\beta-\eta)(\beta-P)} + \beta - P + \frac{\rho}{2} \right)}  \right) \\
\end{align*}
Here, the real parts $\mathrm{Re}(\lambda_{9})$ and $\mathrm{Re}(\lambda_{10})$ are positive when $\beta > \rho/2 + \eta$ and the others are negative for any $\beta$, $\eta$ and $\rho$. 
Additionally, the other eigenvalues are represented as solutions to the following equations:
\begin{align}
    &\left(\frac{\lambda}{\tau}\right)^{3} + (\eta^{3}+2\beta\eta(1+\eta^{2})) \left(\frac{\lambda}{\tau}\right)^{2} -2 \beta \eta^{2} \left(8\beta^{3}+2 \beta(\eta - \beta(1+4\eta)) - \eta^{2}(1+\eta^{2})  \right)P^{4} \frac{\lambda}{\tau} - 8 \beta^{3}(\beta-\eta)\eta^{5} P^{6} = 0 \label{eq:overfit-eq1}\\
    &\left(\frac{\lambda}{\tau}\right)^{3}+\left(\eta P^{4}+2 \eta \beta P^{2}(1+P^{2}) 
 \right) \left(\frac{\lambda}{\tau}\right)^{2} +2\beta \eta^{2} P^{4} \left((P^{2}+P^{4}) + 2\beta^{2}(1+4P) -2 \beta(4\beta^{2}+P)  \right) \frac{\lambda}{\tau} \notag \\
 &~~~~~~~~~~~~~~~~~~~~~~~~~~~~~~~~~~~~~~~~~~~~~~~~~~~~~~~~~~~~~~~~~~~~~~~~~~~~~~~~~~~~~~~~~~~~~~~~~~~~~~~- 8 \beta^{3} \eta^{3} (\beta-P)P^{8}=0 \label{eq:overfit-eq2}
\end{align}
One solution of Eq.~\eqref{eq:overfit-eq1} is positive when $\beta > \eta$, and Eq.~\eqref{eq:overfit-eq1} can be expressed as follows when $\beta = \eta$:
\begin{equation*}
    \frac{\lambda}{\tau}\left(\left(\frac{\lambda}{\tau} \right)^{2} + \frac{\lambda}{\tau} \eta^{2}(2+\eta +2\eta^{2})P^{2}+2\eta^{5}(1+\eta^{2})P^{4}\right)=0,
\end{equation*}
indicating that $\lambda=0$ when $\beta=\eta$. One solution of Eq.~\eqref{eq:overfit-eq2} is positive when $\beta > \eta + \rho$ and Eq.~\eqref{eq:overfit-eq2} can be expressed as follows when $\beta=\eta$:
\begin{equation}
    \frac{\lambda}{\tau}\left(\left(\frac{\lambda}{\tau}\right)^{2}+\frac{\lambda}{\tau} \eta P^{3}(2+P +2P^{2})+2\eta^{2}P^{7}(1+P^{2})\right)=0,
\end{equation}
indicating that $\lambda=0$ when $\beta=\eta+\rho$. Thus, type (2) is stable when $\eta \le \beta \le \eta + \beta$.

\paragraph{Type (3): Learnable Fixed Point} 
\begin{figure}[tb]
    \centering
    \includegraphics[width=\textwidth]{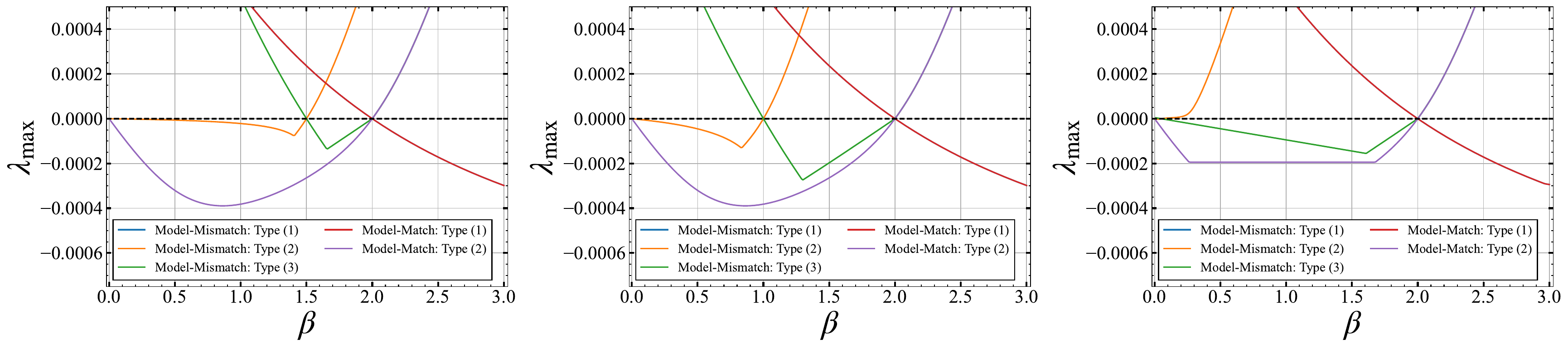}
    \caption{Max eigenvalues of Jacobian for each stable fixed points when $\rho, \eta =0.5, 1.5$ (left), $\rho, \eta = 1.0, 1.0$ (middle) and $\rho, \eta = 1.95, 0.05$ (right) as a function of $\beta$ for both model-matched and model-mismatched cases. }
    \label{fig:max-eigen}
\end{figure}
The fixed point equation has the following solution:
\begin{align*}
    &m^{\ast} = \left(\pm \sqrt{P - \beta}, 0 \right),~ d^{\ast} = \left(\pm \frac{\sqrt{P - \beta}}{P}, 0 \right) \\
    &Q^{\ast} = 
    \begin{pmatrix}
    P-\beta & 0 \\
    0                      & 0
    \end{pmatrix},~
    E^{\ast} =
    \begin{pmatrix}
        \frac{P-\beta}{P^{2}} & 0 \\
        0 & 0 
    \end{pmatrix},~
    R^{\ast} = 
    \begin{pmatrix}
        \frac{P-\beta}{P} & 0 \\
        0 & 0
    \end{pmatrix},~
    D^{\ast} = \left(\frac{\beta}{P}, 1  \right)
\end{align*}
and
\begin{align*}
    &m^{\ast} = \left(0, \pm \sqrt{P - \beta} \right),~ d^{\ast} = \left(0, \pm \frac{\sqrt{P - \beta}}{P} \right) \\
    &Q^{\ast} = 
    \begin{pmatrix}
    0 & 0 \\
    0                      & P - \beta
    \end{pmatrix},~
    E^{\ast} =
    \begin{pmatrix}
        0 & 0 \\
        0 &\frac{P-\beta}{P^{2}} 
    \end{pmatrix},~
    R^{\ast} = 
    \begin{pmatrix}
        0 & 0 \\
        0 & \frac{P-\beta}{\eta+\rho}
    \end{pmatrix},~
    D^{\ast} = \left(1, \frac{\beta}{P}  \right)
\end{align*}
The eigenvalue of the Jacobian can be expressed as follows:
\begin{align*}
    &\frac{\lambda_{1}}{\tau} = -\frac{\beta}{2},~~~\frac{\lambda_{2}}{\tau} = -(1+\eta P)~~~\frac{\lambda_{3}}{\tau} = - (1+\beta \eta), \\
    &\frac{\lambda_{4}}{\tau} = -\left(1+\beta \eta + \sqrt{(1+\beta \eta)^{2}+4\eta(\eta-\beta)}  \right),~~~\frac{\lambda_{5}}{\tau} = -\left(1+\beta \eta - \sqrt{(1+\beta \eta)^{2}+4\eta(\eta-\beta)} \right), \\
    &\frac{\lambda_{6}}{\tau} = - \frac{1}{2} \left(1+\beta P + \sqrt{(1+\beta P)^{2}+4\beta(\beta-P)}  \right),~~~\frac{\lambda_{7}}{\tau} = - \frac{1}{2} \left(1+\beta P - \sqrt{(1+\beta P)^{2}+4\beta(\beta-P)}  \right), \\
    &\frac{\lambda_{8}}{\tau} = - \left(1+\eta P + \sqrt{(1+\eta P)^{2}-4\eta \rho}  \right),~~~\frac{\lambda_{9}}{\tau} = - \left(1+\eta P - \sqrt{(1+\eta P)^{2}-4\eta \rho}  \right), \\
    &\frac{\lambda_{10}}{\tau}  = -\frac{1}{2} \Bigg(2+\eta(\beta+P) + \big((1+\eta \beta)^{2} + (1+\eta P)^{2} + 4\eta(\eta-\rho) + 2\sqrt{\left((1-\beta \eta)^{2}+4\eta^{2} \right)\left((1+\eta P)^{2}-2\eta P\right)}   \big)^{1/2}   \Bigg), \\
    &\frac{\lambda_{11}}{\tau}  = -\frac{1}{2} \Bigg(2+\eta(\beta+P) - \big((1+\eta \beta)^{2} + (1+\eta P)^{2} + 4\eta(\eta-\rho) + 2\sqrt{\left((1-\beta \eta)^{2}+4\eta^{2} \right)\left((1+\eta P)^{2}-2\eta P\right)}   \big)^{1/2}   \Bigg), \\
    &\frac{\lambda_{12}}{\tau}  = -\frac{1}{2} \Bigg(2+\eta(\beta+P) + \big((1+\eta \beta)^{2} + (1+\eta P)^{2} + 4\eta(\eta-\rho) - 2\sqrt{\left((1-\beta \eta)^{2}+4\eta^{2} \right)\left((1+\eta P)^{2}-2\eta P\right)}   \big)^{1/2}   \Bigg), \\
    &\frac{\lambda_{13}}{\tau}  = -\frac{1}{2} \Bigg(2+\eta(\beta+P) - \big((1+\eta \beta)^{2} + (1+\eta P)^{2} + 4\eta(\eta-\rho) - 2\sqrt{\left((1-\beta \eta)^{2}+4\eta^{2} \right)\left((1+\eta P)^{2}-2\eta P\right)}   \big)^{1/2}   \Bigg). \\
\end{align*}
Here, $\lambda_{7}$ is positive when $\beta > \rho+\eta$, $\lambda_{5}$ is positive when $\beta < \eta$, $\lambda_{11}$ is positive when $\beta < \bar{\eta}$ where $\eta < \bar{\eta}$ and the others are negative for any $\beta$, $\eta$ and $\rho$. 
The other eigenvalues are expressed as solutions to the following equation:
\begin{equation*}
    \left(\frac{\lambda}{\tau}\right)^{3} + P^{2}\left(P^{2}+2\beta (1+P^{2})  \right) \left(\frac{\lambda}{\tau}\right)^{2} + 2 \beta P^{4}\left(-2\beta(4\beta^{2}+P) + P^{2}(1+P^{2})+2\beta^{2}(1+2P)  \right) \frac{\lambda}{\tau} - 8 \beta^{3}(\beta-P)P^{8}=0
\end{equation*}
eigenvalue is positive when $\beta > \rho+\eta$, $\beta=\rho+\eta$ and the equation expressed as when $\beta=P$ 
\begin{equation*}
    \frac{\lambda}{\tau}\left(\left(\frac{\lambda}{\tau}\right)^{2} + \frac{\lambda}{\tau} P^{3} (2+ P(1+2P)) + 2P^{7}(1+P^{2})  \right)=0
\end{equation*}
which indicates $\lambda=0$. Thus, type (3) fixed point is stable when $\eta \le \beta \le \rho+\eta$.
Fig.~\ref{fig:max-eigen} presents all types of fixed points and their corresponding maximum eigenvalues as a function of $\beta$.

\subsection{Stability Analysis of Tanh KL Annealing}\label{subsec:stability-tanh-annealing}

\begin{figure}[tb]
    \centering
    \includegraphics[width=\textwidth]{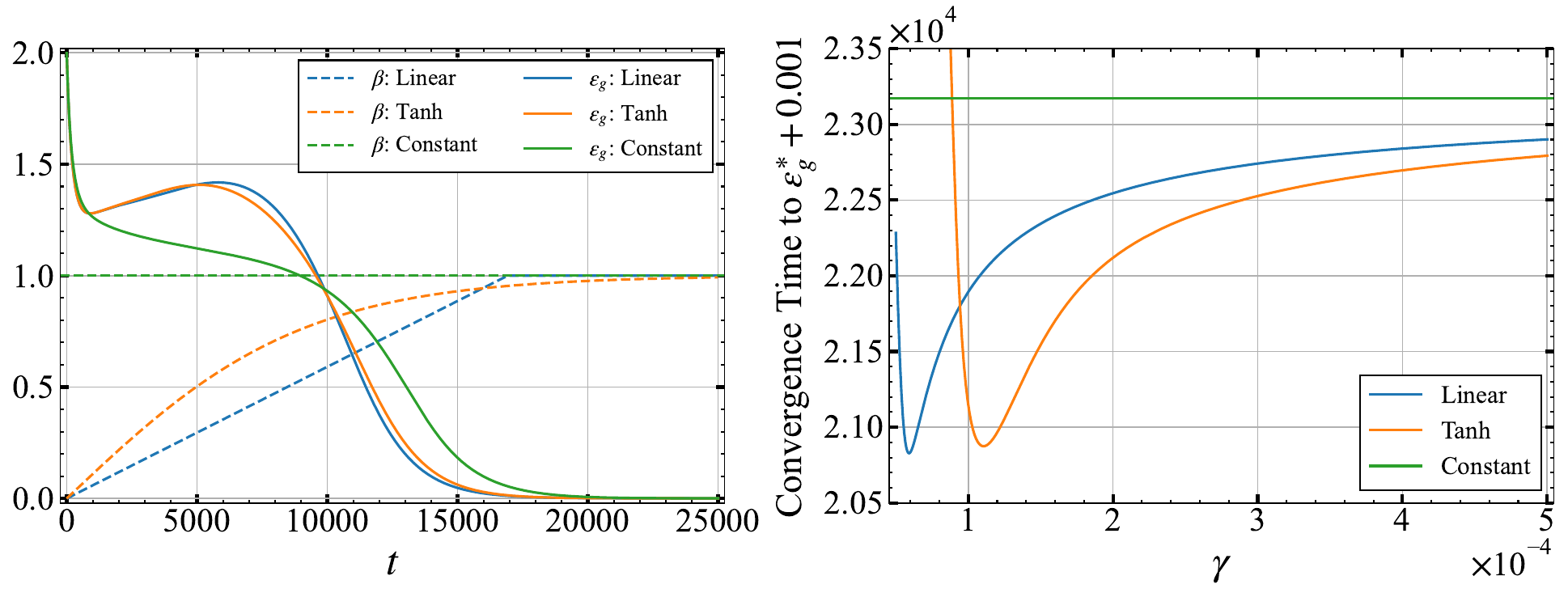}
    \caption{(Top) Time dependence of the generalization error and $\beta$ with linear annealing, tanh KL annealing and constant $\beta=1$ under fixed parameters $\lambda=0$, $\rho=\eta=1$, and $\tau=0.001$. (Bottom) Annealing-rate $\gamma$ dependence of convergence time to the quasi-steady state deviating by $0.001$, i.e., $\varepsilon^{\ast}_{g}+0.001$. The annealing rate $\gamma$ of the learning dynamics with linear annealing, tanh annealing in the top figure is used as the optimal value obtained from the bottom figure.}
    \label{fig:linear-annealing-dynamics}
\end{figure}
For the case of Tanh KL annealing $\beta(t)=\tanh(\gamma t)$, the fixed-point equation can be expressed as follows:
\begin{align*}
\begin{cases}
    F_{m_{11}}(\mac{M}, \beta) = \tau \left(d_{11}(\rho+\eta)-m_{11}(\rho d_{11}^{2} + \eta E_{11} + D_{11}) \right)=0 \\
    F_{d_{11}}(\mac{M}, \beta) =\tau (\rho+\eta) (m_{11} - (Q_{11}+\beta) d_{11}) = 0 \\
    F_{Q_{11}}(\mac{M}, \beta)= 2 \tau \left((\rho m_{11} d_{11} + \eta R_{11}) - Q_{11} (\rho d_{11}^{2} + \eta E_{11} + D_{1})  \right) = 0 \\
    F_{E_{11}}(\mac{M}, \beta) = 2 \tau \left((\rho m_{11}d_{11} + \eta R_{11}) - (Q_{11}+\beta)(\rho d_{11}^{2} + \eta E_{11}) \right) =0 \\
    F_{R_{11}}(\mac{M}, \beta) = \tau \left((1-R_{11})(\rho d_{11}^{2} + \eta E_{11}) - D_{1} R_{11} + (\rho m_{11}^{2} + \eta Q_{11}) - (Q_{11} + \beta)(\rho m_{11} d_{11} + \eta R_{11})  \right)=0 \\
    F_{D_{1}}(\mac{M}, \beta)= \tau \left(\frac{\beta}{D_{1}} - (Q_{11}+\beta)  \right) = 0,\\
    F_{\beta}(\mac{M}, \beta) = \gamma (1-\beta^{2})=0
\end{cases}
\end{align*}
This fixed-point equation has the same stable fixed points as the model-matched case; that is, type (1) posterior collapsed fixed point is stable when $\beta > \eta + \rho$ and type (2) Learnable fixed point is stable when $\beta < \eta + \rho$. 
Additionally, the Jacobian possesses the same eigenvalues as the model-matched case, along with a new eigenvalue of $\lambda_{7}=-2\gamma$ originated from tanh KL annealing. Specifically, for the learnable fixed point, and  excluding $-2\gamma$, the maximal eigenvalue can be expressed as follows when $\rho=2-\nu$ and $\eta = \nu$:
\begin{equation}
    \lambda_{\max}(\nu) =
    \begin{cases}
        \frac{\tau}{2}(\sqrt{5}-3) & \tau(1-2\sqrt{2}+\sqrt{5})/4 \le \nu \le \tau (1+2\sqrt{2} + \sqrt{5})/4 \\
        -\tau(1+2\nu) + \tau\sqrt{1-4\nu(1-4\nu)} &\mathrm{otherwise}
    \end{cases}
\end{equation}
Thus, the conditions under which tanh KL annealing slows down the convergence are expressed as
    \begin{equation*}
        \gamma \le  
        \begin{cases}
            \frac{\tau}{4}(3-\sqrt{5}),~\tau (1-2\sqrt{2}+\sqrt{5})/4 \le \nu \le \tau (1+2\sqrt{2}+\sqrt{5})/4 \\
            \tau \left(\nu +\frac{1}{2}\right) - \tau \sqrt{\nu(2\nu-1) +\frac{1}{4}},~\mathrm{otherwise}.
        \end{cases}
    \end{equation*}
\section{Additional Results}
\subsection{Linear Annealing}\label{subsec:linear-annealing}
In this section, we demonstrate the properties of the linear annealing $\beta(t)=\gamma t$ which is used in various applications. Fig.~\ref{fig:linear-annealing-dynamics} demonstrates the generalization error as a function of $t$ for both Linear and tanh KL annealing using the optimal annealing rate and for the constant $\beta=1$. It also demonstrates the $\gamma$ dependency of the quasi-steady-state convergence times for linear and tanh KL annealing and constant $\beta=1$. As a result, this experiment demonstrates that both linear KL annealing and tanh KL annealing exhibit qualitatively similar behavior.

\end{document}